\pdfoutput=1

\documentclass[11pt]{article}

\usepackage[final]{ACL2023}

\usepackage{times}
\usepackage{latexsym}

\usepackage[T1]{fontenc}

\usepackage[utf8]{inputenc}

\usepackage{microtype}

\usepackage{inconsolata}

\usepackage{amstext}
\usepackage{amsmath}
\usepackage{amssymb}
\usepackage{subcaption}
\usepackage{graphicx}
\usepackage{booktabs}
\usepackage{multirow}
\usepackage{tikz-dependency}
\usepackage{amsthm}
\usepackage{bbm}
\newtheorem{lemma}{Lemma}

\title{Probabilistic Transformer: A Probabilistic Dependency Model for Contextual Word Representation}

\author{Haoyi Wu \and Kewei Tu\thanks{\; Corresponding author.} \\
School of Information Science and Technology, ShanghaiTech University \\
Shanghai Engineering Research Center of Intelligent Vision and Imaging \\
\texttt{\{wuhy1, tukw\}@shanghaitech.edu.cn}}

\begin{document}
\maketitle
\begin{abstract}

  Syntactic structures used to play a vital role in natural language processing (NLP), but since the deep learning revolution, NLP has been gradually dominated by neural models that do not consider syntactic structures in their design. One vastly successful class of neural models is transformers. When used as an encoder, a transformer produces contextual representation of words in the input sentence. In this work, we propose a new model of contextual word representation, not from a neural perspective, but from a purely syntactic and probabilistic perspective. Specifically, we design a conditional random field that models discrete latent representations of all words in a sentence as well as dependency arcs between them; and we use mean field variational inference for approximate inference. Strikingly, we find that the computation graph of our model resembles transformers, with correspondences between dependencies and self-attention and between distributions over latent representations and contextual embeddings of words. Experiments show that our model performs competitively to transformers on small to medium sized datasets. We hope that our work could help bridge the gap between traditional syntactic and probabilistic approaches and cutting-edge neural approaches to NLP, and inspire more linguistically-principled neural approaches in the future.\footnote{Our code is publicly available at \url{https://github.com/whyNLP/Probabilistic-Transformer}}

\end{abstract}

\section{Introduction}

Once upon a time, syntactic structures were deemed essential in natural language processing (NLP). Modeling and inference about syntactic structures was an indispensable component in many NLP systems. That has all changed since the deep learning revolution started a decade ago. Modern NLP predominantly employs various neural models, most of which do not consider syntactic structures in their design.

One type of neural models that are particularly successful is transformers \cite{vaswani2017attention}. Given an input text, a transformer produces a vector representation for each word that captures the meaning as well as other properties of the word in its context. Such contextual word representations can then be served into downstream neural networks for solving various NLP tasks. The power of transformers in producing high-quality contextual word representations is further unleashed with large-scale pretraining \cite{devlin2019bert, liu2020roberta}. Nowadays, a vast majority of NLP models and systems are built on top of contextual word representations produced by some variants of pretrained transformers.

Like most other neural models, transformers were developed based on human insight and trial and error, without explicit design for incorporating syntactic structures. Nevertheless, there is evidence that contextual word representations produced by pretrained transformers encode certain syntactic structures \cite{hewitt-manning-2019-structural, tenney2018what} and attention heads in pretrained transformers may reflect syntactic dependencies \cite{clark-etal-2019-bert, htut2019attention, ravishankar-etal-2021-attention}. Because of the heuristic nature of the transformer model design, exactly how transformers acquire such syntactic capability remains unclear.

In this paper, we propose \emph{probabilistic transformers}, a very different approach to deriving contextual word representations that is based on classic non-neural probabilistic modeling with innate syntactic components. Specifically, we design a conditional random field that models discrete latent representations of all words as well as a syntactic dependency structure of the input sentence, and we define a potential function which evaluates the compatibility of the latent representations of any pair of words connected by a dependency arc. We use mean field variational inference for approximate inference, producing a marginal distribution for each latent word representation, the probability vector of which can then be used as a contextual vector representation of the word. 

While we propose our model from a purely syntactic and probabilistic perspective that is unrelated to transformers, we show that there is a striking resemblance between the computation graph of the inference procedure of our model and that of a transformer, with our intermediate distributions over dependency heads corresponding to self-attention scores and our intermediate distributions over latent word representations corresponding to intermediate word embeddings in a transformer. In short, we start with a probabilistic syntactic model but reach the transformer!
We empirically compare our model with transformers when trained with either masked language modeling or downstream tasks. Our experimental results show that our model performs competitively to transformers on small to medium sized datasets.

We hope that probabilistic transformers, instead of being a replacement of transformers, could benefit the analysis of the syntactic capability of transformers and at the same time inspire novel extensions of transformers. Furthermore, we hope our work would promote future research of neural models that are linguistically more principled, theoretically more well-founded, and empirically no less powerful than existing models.

\section{Probabilistic Transformers}

We will first introduce the basic model, a conditional random field (CRF) as illustrated in Figure~\ref{fig:factor-graph}, then show the inference procedure, and finally introduce some variants to the basic model.

\subsection{The CRF Model}
\label{sec:crf-model}

Given a sentence (a sequence of words), denote $n$ as the sequence length. For the $i$-th word, we define $Z_i$ as a discrete latent label that represents the syntactic (and possibly semantic) property of the word in the sentence (i.e., it is a contextual representation) with a label set of size $d$. Such a discrete representation deviates from the common practice of representing a word with a continuous vector, but it is sufficient at least for syntactic processing \cite{kitaev-etal-2022-learned} and it greatly simplifies our probabilistic model.
For the $i$-th word, we also define $H_i \in \{1, 2, \cdots, n\}$ representing the syntactic dependency head of the word. So the set of variables $\{H_i\}_{i=1}^n$ specifies a dependency structure. We may also allow $H_i$ to point to a dummy root node, which will be discussed in Section~\ref{sec:var-rn}. We follow the head-selection paradigm of dependency parsing and do not enforce the tree constraint, which again simplifies our model design.

\begin{figure}
  \centering
  \includegraphics[page=2,width=.5\textwidth,trim=80 200 300 90,clip]{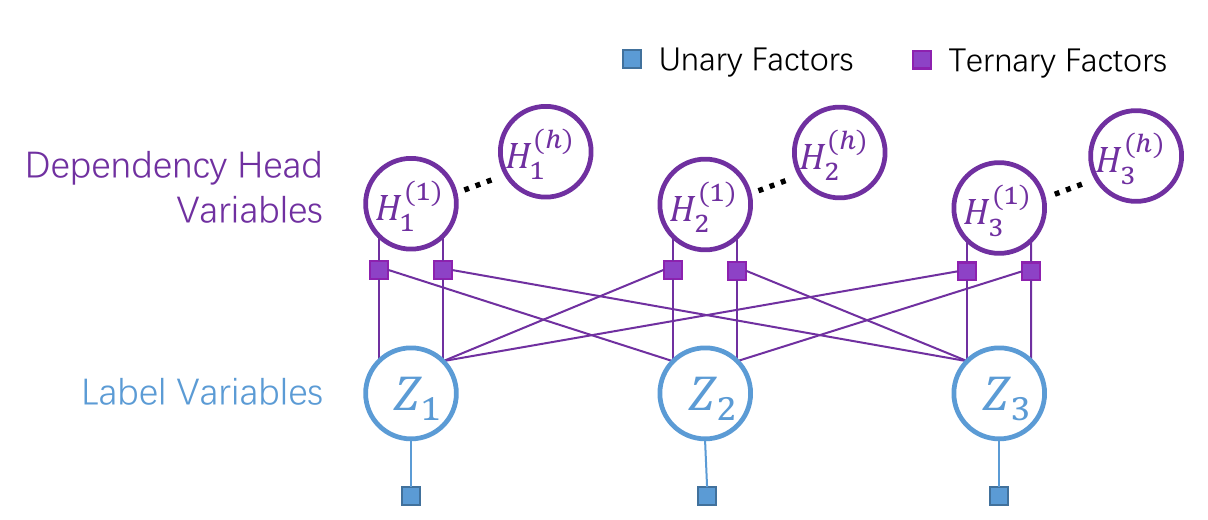}
  \caption{The factor graph for our CRF model with $n=3$. For clarity, ternary factors that connect to $H_i^{(c)}$ with $c>1$ are not shown in the figure.}
  \label{fig:factor-graph}
\end{figure}

Next, we define two types of potential functions. For the $i$-th word $w_i$, we define a unary potential function (corresponding to the unary factors in Figure~\ref{fig:factor-graph}) evaluating the compatibility of the word and its label $Z_i$:
\begin{equation}
  \phi_u(Z_i) = \exp\left(\mathbf{S}_{w_i, Z_i}\right) 
\end{equation}
where $\mathbf{S} \in \mathbb{R}^{|\mathcal{V}| \times d}$ is a score matrix, $|\mathcal{V}|$ is the size of the vocabulary. For simplicity, we do not exploit any morphological or contextual features for computing the scores.
For every pair of words $w_i$ and $w_j$ ($i \neq j$), we define a ternary potential function (corresponding to the ternary factors in Figure~\ref{fig:factor-graph}) over $Z_i$, $Z_j$ and $H_i$, which evaluates the compatibility between the labels of the two words if $w_j$ is the dependency head of $w_i$:
\begin{equation} \label{eq:ter-potential}
  \begin{aligned}
  \phi_t(&H_i, Z_i, Z_j) =\\
  &\left\{
    \begin{array}{lcl}
      \exp\left(\mathbf{T}_{Z_i,Z_j}\right)       &      & {H_i=j}\\
      1      &      & {\text{otherwise}}
    \end{array} 
  \right.
  \end{aligned}
\end{equation}
where $\mathbf{T} \in \mathbb{R}^{d \times d}$ is a score matrix.

Inspired by the multi-head structure in transformers, we allow multiple dependency structures for the same sentence, which may represent different flavors of dependencies. Each dependency structure resides in a different \emph{channel} with its own dependency head variables and ternary potential functions. For the $c$-th channel, we denote the set of dependency head variables by $\{H_i^{(c)}\}_{i=1}^n$ and the score matrix of the ternary potential function by $\mathbf{T}^{(c)}$. Let $h$ denote the total number of channels. We may stack all the score matrices $\mathbf{T}^{(c)}$ for $c=1,\cdots,h$ to form a score tensor $\mathbf{T} \in \mathbb{R}^{d \times d \times h}$. Note that all the channels share the same set of latent label variables $\{Z_i\}_{i=1}^n$.



\subsection{Inference}
\label{sec:inference}

Following \citet{wang-tu-2020-second}, we use Mean Field Variational Inference (MFVI) to perform approximate inference. Different from the previous work, however, we need to run inference over latent labels in addition to dependency heads.


MFVI iteratively passes messages between random variables and computes an approximate posterior marginal distribution for each random variable (denoted by $Q(\cdot)$).
Let $\mathcal{F}_{ic}^{(t)}$ denote the message received by variable $H_i^{(c)}$ at time step $t$ from ternary factors, and $\mathcal{G}_{i}^{(t)}$ denote the message received by variable $Z_i$ at time step $t$ from ternary factors. We have
\begin{align}
  \label{eq:msg-f} &\mathcal{F}_{ic}^{(t)}(j) = \sum_a \sum_b \left( Q_i^{(t)}(a)Q_j^{(t)}(b)\mathbf{T}_{a,b}^{(c)} \right) \\
  \label{eq:msg-g} &\begin{aligned}
    \mathcal{G}_{i}^{(t)}(a) = \sum_c \sum_{j \neq i} &\sum_b \left( Q_{ic}^{(t)}(j)Q_j^{(t)}(b)\mathbf{T}_{a,b}^{(c)}\right.\\
    &\left. + Q_{jc}^{(t)}(i)Q_j^{(t)}(b)\mathbf{T}_{b,a}^{(c)} \right)
  \end{aligned}
\end{align}
where
\begin{align}
  \label{eq:msg-qz} Q_i^{(t)}(a) &\propto \exp\left(\mathbf{S}_{w_i, a} + \mathcal{G}_{i}^{(t-1)}(a)\right) \\
  \label{eq:msg-qh} Q_{ic}^{(t)}(j) &\propto \exp\left(\mathcal{F}_{ic}^{(t-1)}(j)\right)
\end{align}
are the approximate marginal distributions at time step $t$, with $Q_i^{(t)}(\cdot)$ over $Z_i$ and $Q_{ic}^{(t)}(\cdot)$ over $H_i^{(c)}$. We initialize these distributions by
\begin{align}
  \label{eq:init-qz} Q_i^{(0)}(a) &\propto \exp\left(\mathbf{S}_{w_i, a}\right) \\
  \label{eq:init-qh} Q_{ic}^{(0)}(j) &\propto 1
\end{align}

After a fixed number of $T>0$ iterations, we obtain the final posterior marginal distribution $Q_i^{(T)}(Z_i)$ for $i=1,\cdots,n$. Resulted from interactions with all the words of the sentence, the distribution $Q_i^{(T)}(Z_i)$ incorporates information of not only the $i$-th word, but also its context. Therefore, we can treat the probability vector of this distribution as a contextual vector representation for the $i$-th word. In practice, we find that using unnormalized scores in log space as contextual word representations produces better results, i.e., we skip exponentiation and normalization when computing $Q_i^{(T)}(Z_i)$ using Equation~\ref{eq:msg-qz} during the final iteration.

Since all the computation during MFVI is fully differentiable, we can regard the corresponding computation graph as a recurrent or graph neural network parameterized with score matrix $\mathbf{S}$ and tensor $\mathbf{T}$.
We can use the contextual word representations for downstream tasks by connecting the network to any downstream task-specific network, and we can update the model parameters using any task-specific learning objective through gradient descent. This is exactly the same as how transformers are used.

\subsection{Extensions and Variants}
\label{sec:variants}

We introduce a few extensions and variants to the basic model that are empirically beneficial. Additional variants are discussed in Appendix~\ref{apx:more-variants}.

\subsubsection{Distance}
\label{sec:distance}

Similar to the case of transformers, our probabilistic model is insensitive to the word order of the input sentence. In order to capture the order information, we apply relative positional encoding to our model by using distance-sensitive ternary potential functions. Specifically, we use different ternary scores for different distances between words denoted by the two $Z$ variables of the potential function. The ternary potential function in Equation~\ref{eq:ter-potential} becomes:
\begin{equation} \label{eq:distance}
  \begin{aligned}
  &\phi_t(H_i^{(c)}, Z_i, Z_j) =\\
  &\left\{
    \begin{array}{lcl}
      \exp\left(\mathbf{T}[f(i-j)]_{Z_i,Z_j}^{(c)}\right)       &      & {H_i^{(c)}=j}\\
      1      &      & {\text{otherwise}}
    \end{array} 
  \right.
  \end{aligned}
\end{equation}
where $f$ is a clip function with threshold $\gamma$:
\begin{equation} \label{eq:clip}
  f(x) = \left\{
    \begin{array}{lcl}
      0 & & {x < -\gamma}\\
      x + \gamma + 1 & & {-\gamma \leq x < 0}\\
      x + \gamma     & & {0 < x \leq \gamma}\\
      2\gamma + 1      &      & {x > \gamma}
    \end{array} 
  \right.
\end{equation}
Notice that $x$ cannot be zero since the head of a word cannot be itself. We set $\gamma = 3$ by default.

\subsubsection{Asynchronous Update}
\label{sec:var-au}

During inference of the basic model, we iteratively update all variables in a synchronous manner. This can be problematic. Consider the first iteration. The messages passed to $Z$ variables from $H$ variables do not contain meaningful information because the initial distributions over $H$ are uniform. Consequently, after one iteration, distributions over all $Z$ variables become almost identical.

To fix this problem, we use the asynchronous update strategy by default in this work. For each iteration, we first update distributions over $H$ variables, and then update distributions over $Z$ variables based on the updated distributions over $H$ variables. Formally, we rewrite Formula~\ref{eq:msg-qh} as
\begin{equation*}
  Q_{ic}^{(t)}(j) \propto \exp\left(\mathcal{F}_{ic}^{(t)}(j)\right)
\end{equation*}
and eliminate Formula~\ref{eq:init-qh} because distributions over $H$ variables no longer need initialization.

\subsubsection{Message Weight}
\label{sec:var-rw}

During inference, $H$ variables have much fewer message sources than $Z$ variables. This often pushes $H$ variables towards being uniformly distributed. To balance the magnitude of the messages, we follow the Entropic Frank-Wolfe algorithm \cite{le2021regularized}, a generalization of MFVI, and introduce weight $\lambda_Z > 0$ and $\lambda_H > 0$ to Equation~\ref{eq:msg-qz} and \ref{eq:msg-qh}:
\begin{align}
  \label{eq:reg-weight-z} Q_i^{(t)}(a) &\propto \exp\left(\frac{1}{\lambda_Z}\left(\mathbf{S}_{w_i,a} + \mathcal{G}_{i}^{(t-1)}(a)\right)\right) \\
  \label{eq:reg-weight-h} Q_{ic}^{(t)}(j) &\propto \exp\left(\frac{1}{\lambda_H}\mathcal{F}_{ic}^{(t-1)}(j)\right)
\end{align}
We set $\lambda_Z = 1$ and $\lambda_H = \frac{1}{d}$ by default\footnote{We choose these weights in a similar way to choosing the scaling factor in scaled dot-product attention of transformers. See more details in Appendix~\ref{apx:choice}.}.


\subsubsection{Tensor Decomposition}
\label{sec:var-td}

Ternary score $\mathbf{T}$ is a tensor of shape $d \times d \times h$. Since $d$ is usually set to several hundred, such a tensor leads to a huge number of parameters. To reduce the number of parameters, we apply the Kruskal form (which is closely related to tensor rank decomposition) to build the ternary score from smaller tensors.
\begin{equation} \label{eq:td-uvw}
  \mathbf{T}_{a,b}^{(c)} = \sum_{l=1}^r \mathbf{U}_{a,l} \cdot \mathbf{V}_{b,l} \cdot \mathbf{W}_{c,l}
\end{equation}
where $\mathbf{U}, \mathbf{V} \in \mathbb{R}^{d \times r}$ and $\mathbf{W} \in \mathbb{R}^{h \times r}$.

Since the number of channels $h$ is relatively small, we may also choose only to decompose the first two dimensions.
\begin{equation} \label{eq:td-uv}
  \mathbf{T}_{a,b}^{(c)} = \sum_{l=1}^r \mathbf{U}_{a,c,l} \cdot \mathbf{V}_{b,c,l}
\end{equation}
where $\mathbf{U}, \mathbf{V} \in \mathbb{R}^{d \times h \times r}$.

\subsubsection{Root Node}
\label{sec:var-rn}

Dependency parsing assumes a dummy root node, which we may add to the CRF model. The root node is not associated with any word and instead can be seen as representing the entire sentence. Therefore, we assume that it has a different (and possibly larger) label set from words and hence requires a different ternary potential function. Specifically, we define $Z_{ROOT}$ as a discrete latent label of the root node with a label set of size $d_{root}$. For $i \in \{1,2,\cdots,n\}, c \in \{1,2,\cdots,h\}$, we add a ternary potential function over $Z_i, H_i^{(c)}$ and $Z_{ROOT}$: 
\begin{equation*}
  \begin{aligned}
  \phi_t(H_i^{(c)} &, Z_i, Z_{ROOT}) =\\
  &\left\{
    \begin{array}{lcl}
      \exp\left(\mathbf{T}_{Z_i, Z_{ROOT}}^{'(c)}\right)       &      & {H_i^{(c)}=ROOT}\\
      1      &      & {\text{otherwise}}
    \end{array} 
  \right.
  \end{aligned}
\end{equation*}
where $\mathbf{T}^{'} \in \mathbb{R}^{d \times d_{root} \times h}$ is the root score tensor. During inference, we initialize $Q^{(0)}(Z_{ROOT})$ with a uniform distribution. After inference, we can regard the posterior marginal distribution of $Z_{ROOT}$ as a sentence representation.

\section{Comparison with Transformers}
\label{sec:cp-to-trans}

Although our probabilistic transformers are derived as a probabilistic model of dependency structures over latent word labels, we find that its computational process has lots of similarities to that of transformers. Below, we first re-formulate a probabilistic transformer in a tensor form to facilitate its comparison with a transformer, and then discuss the similarities between the two models at three levels.

\subsection{Probabilistic Transformers in Tensor Form}


Consider a probabilistic transformer using a distance-insensitive ternary potential function without a dummy root node. We tensorize the update formulas in the inference process of probabilistic transformers. Suppose $Q_z^{(t)} \in \mathbb{R}^{n \times d}$ is a tensor that represents the posterior distributions of all the $Z$ variables, and $Q_{h,c}^{(t)} \in \mathbb{R}^{n \times n}$ is a tensor that represents the posterior distributions of all the $H$ variables in channel $c$ (with a zero diagonal to rule out self-heading). We can rewrite Equation~\ref{eq:msg-f} and \ref{eq:msg-g} as
\begin{align}
  \label{eq:msg-tensor-f} \mathcal{F}_{c}^{(t)} &= Q_z^{(t)} \mathbf{T}^{(c)} Q_z^{(t)T} \\
  \label{eq:msg-tensor-g-original} \mathcal{G}^{(t)} &= \sum_c \left( Q_{h,c}^{(t)} Q_z^{(t)} \mathbf{T}^{(c)T}+ Q_{h,c}^{(t)T} Q_z^{(t)} \mathbf{T}^{(c)} \right)
\end{align}
where
\begin{align}
  \label{eq:msg-tensor-qz} Q_z^{(t)} &= \sigma(\mathbf{S} + \mathcal{G}^{(t-1)}) \\
  \label{eq:msg-tensor-qh-novar} Q_{h,c}^{(t)} &= \sigma(\mathcal{F}_{c}^{(t-1)})
\end{align}
and $\sigma$ is the softmax function. We still set $\lambda_Z$ to its default value $1$ but regard $\lambda_H$ as a hyperparameter. With asynchronous update, Equation~\ref{eq:msg-tensor-qh-novar} becomes:
\begin{equation} \label{eq:msg-tensor-qh}
  Q_{h,c}^{(t)} = \sigma\left(\frac{\mathcal{F}_{c}^{(t)}}{\lambda_H}\right)
\end{equation}


We assume that $\mathbf{T}^{(c)}$ is symmetric for $c=1,\cdots,h$. This is the only assumption that we make in this section beyond the original definition from the previous section. Symmetric score matrices indicate that the ternary factors are insensitive to the head-child order, which is related to undirected dependency parsing \cite{sleator-temperley-1993-parsing}. If $\mathbf{T}^{(c)}$ is symmetric, then $Q_{h,c}^{(t)}$ is also symmetric based on Formula~\ref{eq:msg-tensor-f} and \ref{eq:msg-tensor-qh}. Thus, we can simplify Equation~\ref{eq:msg-tensor-g-original} to
\begin{equation} \label{eq:msg-tensor-g}
  \mathcal{G}^{(t)} = 2 \sum_c Q_{h,c}^{(t)} Q_z^{(t)} \mathbf{T}^{(c)T}
\end{equation}

Suppose we decompose the ternary score tensor into two tensors $\mathbf{U}, \mathbf{V} \in \mathbb{R}^{d \times h \times r}$ according to Equation~\ref{eq:td-uv}, which can be rewritten as:
\begin{equation} \label{eq:score-td}
  \mathbf{T}^{(c)} = \mathbf{U}^{(c)} \mathbf{V}^{(c)T}
\end{equation}
where $\mathbf{U}^{(c)}, \mathbf{V}^{(c)} \in \mathbb{R}^{d \times r}$ are the $c$-th channel of tensor $\mathbf{U}$ and $\mathbf{V}$ respectively. Substitute \ref{eq:score-td} into \ref{eq:msg-tensor-f} and \ref{eq:msg-tensor-g}, we have
\begin{align}
  \label{eq:msg-tensor-uv-f} \mathcal{F}_{c}^{(t)} &= Q_z^{(t)} \mathbf{U}^{(c)} \mathbf{V}^{(c)T} Q_z^{(t)T} \\
  \label{eq:msg-tensor-uv-g} \mathcal{G}^{(t)} &= 2 \sum_c Q_{h,c}^{(t)} Q_z^{(t)} \mathbf{V}^{(c)} \mathbf{U}^{(c)T}
\end{align}


We define
\begin{align}
  Q_c &= Q_z^{(t-1)} \mathbf{U}^{(c)} \\
  K_c = V_c &= Q_z^{(t-1)} \mathbf{V}^{(c)}
\end{align}
For time step $t-1$, we could rewrite Formula~\ref{eq:msg-tensor-uv-f} and \ref{eq:msg-tensor-uv-g} as
\begin{align}
  \label{eq:msg-tensor-f-in-qkv} \mathcal{F}_{c}^{(t-1)} &= Q_c K_c^T \\
  \label{eq:msg-tensor-g-in-qkv} \mathcal{G}^{(t-1)} &= 2 \sum_c Q_{h,c}^{(t-1)} V_c \mathbf{U}^{(c)T}
\end{align}

Apply Equation~\ref{eq:msg-tensor-g-in-qkv}, \ref{eq:msg-tensor-qh}, \ref{eq:msg-tensor-f-in-qkv} to \ref{eq:msg-tensor-qz}, we have
\begin{equation} \label{eq:final-update}
  Q_z^{(t)} = \sigma(\mathbf{S} + 2\sum_c \operatorname{channel}_c \mathbf{U}^{(c)T})
\end{equation}
where 
\begin{equation} \label{eq:channel-c}
  \operatorname{channel}_c = \sigma\left(\frac{Q_c K_c^T}{\lambda_H}\right) V_c
\end{equation}
We call the computation of $\operatorname{channel}_c$ a \emph{single-channel update} for channel $c$.

Now we have a tensorized formulation of the computation in probabilistic transformers and we are ready for its comparison with transformers at three different levels.


\subsection{Single-Channel Update vs. Scaled Dot-Product Attention}
\label{sec:single-channel}

Scaled dot-product attention in transformers is formulated as:
\begin{equation*}
  \operatorname{Attention}(Q, K, V)=\sigma\left(\frac{Q K^{T}}{\sqrt{d_{k}}}\right) V
\end{equation*}
As we can see, our single-channel update in Equation~\ref{eq:channel-c} is almost identical to scaled dot-product attention in transformers. The only difference is that the diagonal of the tensor $Q_c K_c^T$ is zero in our model because the head of a word cannot be itself. 

\subsection{Multi-Channel Update vs. Multi-Head Attention}
\label{sec:cmp-multi-head}

\begin{figure*}[tb]
  \centering
  \begin{subfigure}{.4\textwidth}
    \centering
    \includegraphics[page=3,width=\textwidth,trim=60 20 510 200,clip]{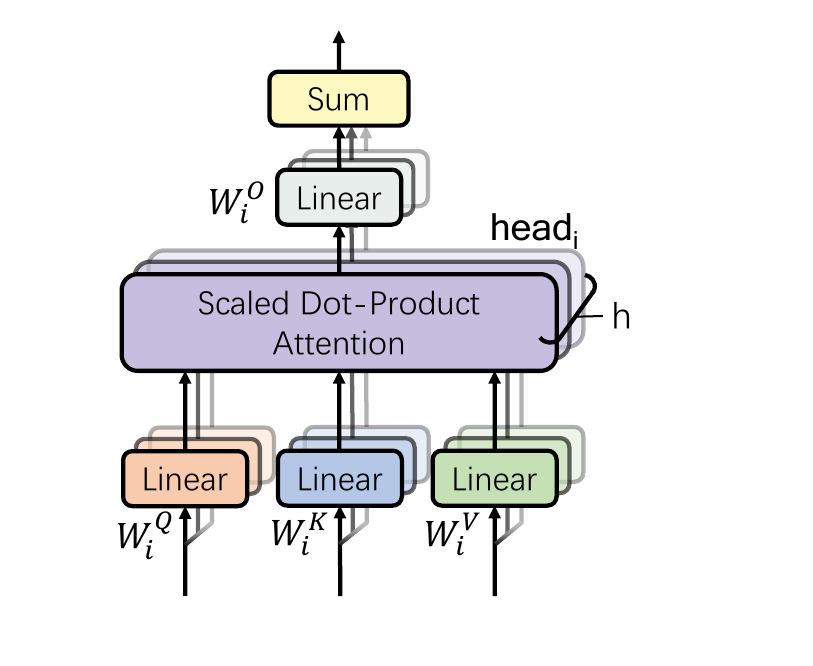}
    \caption{Multi-Head Attention}
  \end{subfigure}
  \quad \quad \quad
  \begin{subfigure}{.4\textwidth}
    \centering
    \includegraphics[page=4,width=\textwidth,trim=500 20 60 200,clip]{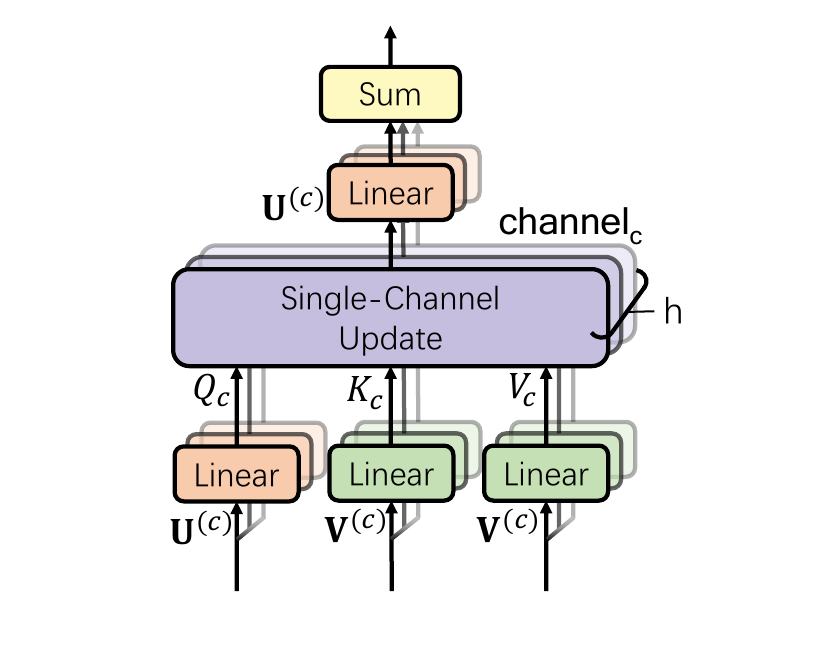}
    \caption{Multi-Channel Update}
    \label{subfig:multi-channel}
  \end{subfigure}
  \caption{Computation graphs for multi-head attention in transformers and multi-channel update in probabilistic transformers. See an explanation of replacing concat+linear with linear+sum 
  in the upper part of multi-head attention in Section~\ref{sec:cmp-multi-head}.}
  \label{fig:cmp-multi-channel}
\end{figure*}

Multi-head attention in transformers is formulated as:
\begin{equation*}
  \begin{aligned}
    &\operatorname{MultiHead}(Q, K, V)= \\
    &\operatorname{Concat}\left(\operatorname{head}_{1}, \ldots, \operatorname{head}_{h}\right) W^{O}
  \end{aligned}
\end{equation*}
where 
\begin{equation*}
  \operatorname{head}_{i}=\operatorname{Attention}\left(Q W_{i}^{Q}, K W_{i}^{K}, V W_{i}^{V}\right)
\end{equation*}
It is equivalent to
\begin{equation*}
  \operatorname{MultiHead}(Q, K, V)=\sum_i \operatorname{head}_{i} (W_i^{O})^{T}
\end{equation*}
where $W^{O} \equiv \operatorname{Concat}(W_{1}^{O}, \ldots, W_{h}^{O})$ and $W_i^{Q}, W_i^{K}, W_i^{V}, W_i^{O} \in \mathbb{R}^{d \times r}$. Our multi-channel update formula (the second term within the softmax function in Equation~\ref{eq:final-update}) is similar to the multi-head attention in transformers, as shown in Figure~\ref{fig:cmp-multi-channel}. The main difference is that probabilistic transformers use the same parameters for $W^K$ and $W^V$ (both are $\mathbf{V}$, shown in green color in Figure~\ref{subfig:multi-channel}) and for $W^Q$ and $W^O$ (both are $\mathbf{U}$, shown in orange color in Figure~\ref{subfig:multi-channel}). 

Recall that $\mathbf{U}$ and $\mathbf{V}$ are obtained from matrix decomposition (Equation~\ref{eq:td-uv}). Therefore, the correspondence between $\mathbf{U}$, $\mathbf{V}$ and $W^{Q}, W^{K}, W^{O}, W^{V}$ in transformers suggests that the latter can also be seen as derived from tensor decomposition. Previous work on transformers has the same findings \cite{elhage2021mathematical}.


\subsection{Full Model Comparison}
\label{sec:cmp-whole-block}

Figure~\ref{fig:cmp-whole-block} compares the full computation graphs of the two models, which have a similar overall structure that repeats a module recurrently until outputting contextual word representations. Within the module, we have also established the correspondence between multi-channel update and multi-head attention. On the other hand, there are a few interesting differences.

First, our model does not have a feed-forward structure as in a transformer. However, we do propose a variant of our model that contains global variables representing topics (Appendix~\ref{apx:global-vars}), which may have similar functionality to the feed-forward structure.

Second, our model does not have residual connections or layer norms. Instead, it adds the initial distributions (unary scores) to the updated message at each iteration. This may replace the functionality of residual connections and may even make more sense when the downstream task strongly depends on the original word information.

Third, we have an additional softmax in each iteration. Note that we do softmax before the first iteration (Equation~\ref{eq:init-qz}) and also at the end of each iteration (Equation~\ref{eq:final-update}), but bypass it in the last iteration when producing the output word representations, so our model could be equivalently formulated as doing softmax before each iteration, which we show in Figure~\ref{subfig:whole-block}. Doing softmax in this way is similar to the layer norm in pre-LN transformers \cite{xiong2020layer} (Figure~\ref{subfig:pre-ln-tsfm}).

Finally, our model shares parameters in all iterations. This is similar to some variants of transformers that share parameters between layers, such as Universal Transformer \cite{dehghani2018universal} and ALBERT \cite{lan2019albert}.

One consequence of these differences is that probabilistic transformers have much fewer parameters than transformers with the same number of layers, heads and embedding dimensions, because of shared parameters between iterations, absence of a feed-forward structure, and tied parameter matrices in multi-channel updates.





\begin{figure*}[tb]
  \centering
  \begin{subfigure}{.286\textwidth}
    \centering
    \includegraphics[page=5,width=\textwidth,trim=30 0 620 0,clip]{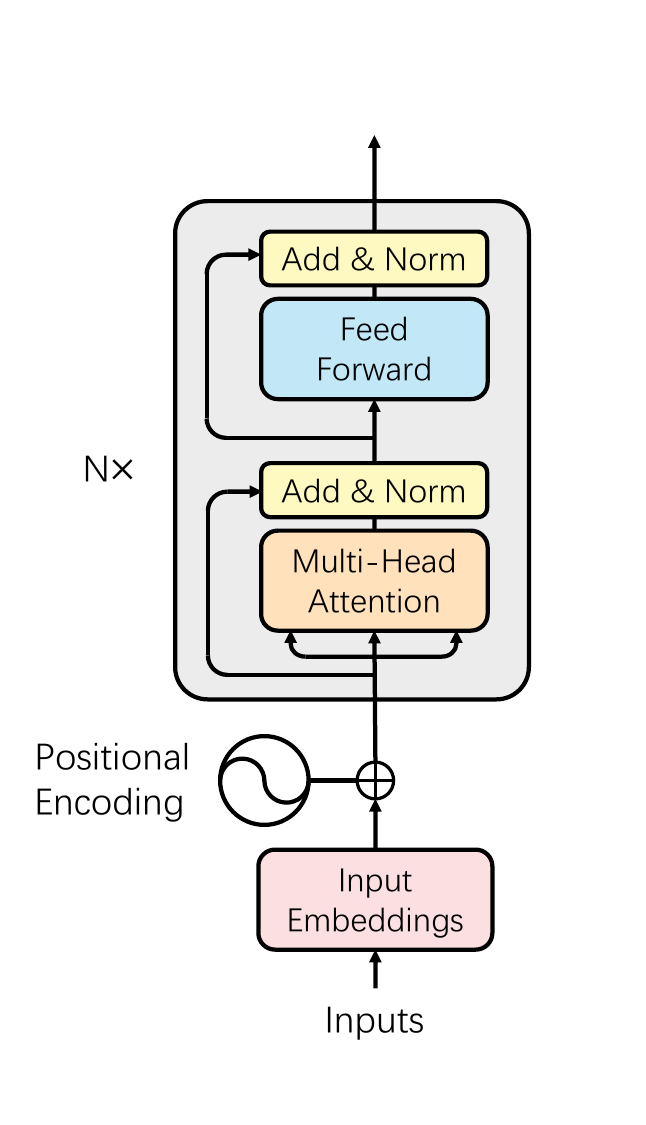}
    \caption{Transformer}
  \end{subfigure}
  \quad 
  \begin{subfigure}{.286\textwidth}
    \centering
    \includegraphics[page=6,width=\textwidth,trim=320 0 330 0,clip]{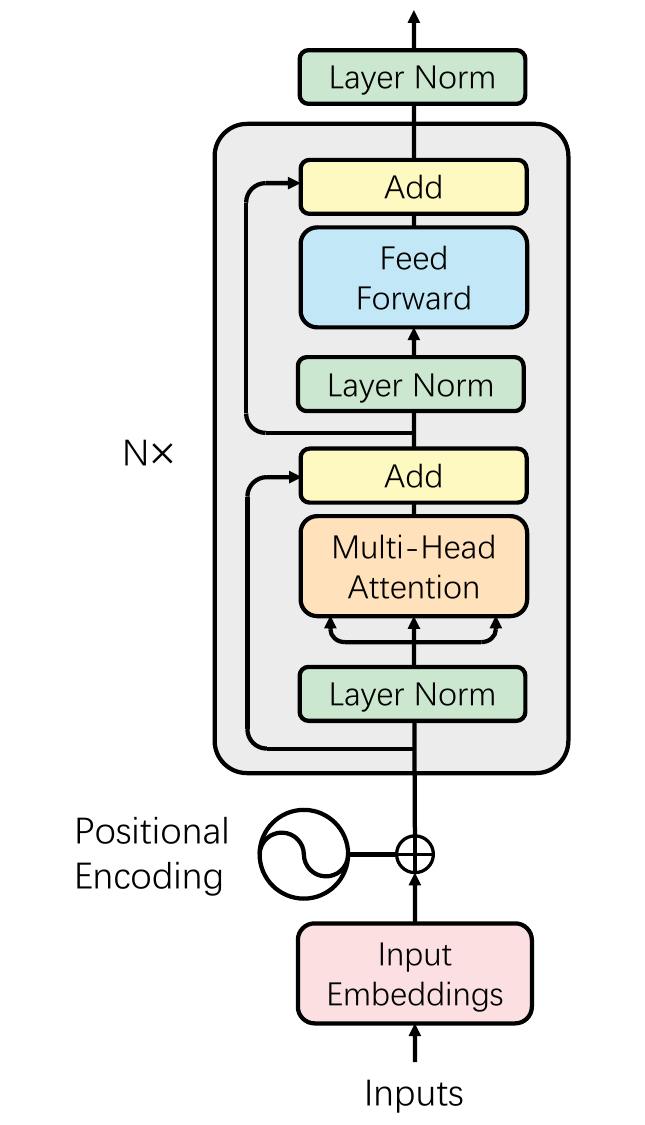}
    \caption{Pre-LN Transformer}
    \label{subfig:pre-ln-tsfm}
  \end{subfigure}
  \quad 
  \begin{subfigure}{.286\textwidth}
    \centering
    \includegraphics[page=7,width=\textwidth,trim=630 0 20 0,clip]{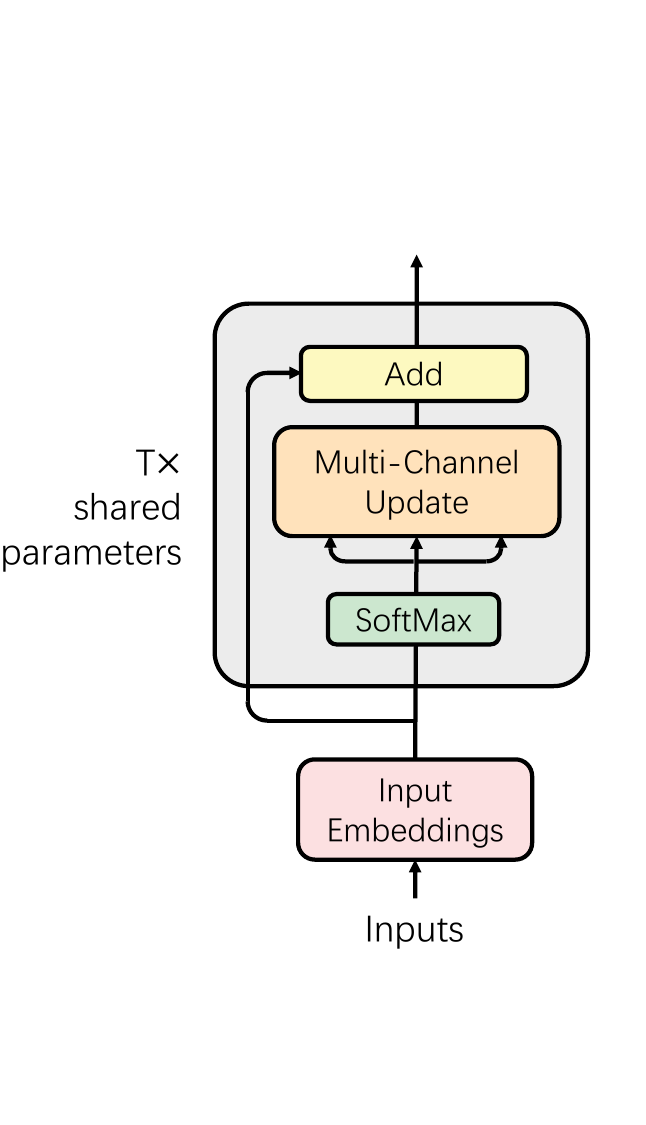}
    \caption{Probabilistic Transformer}
    \label{subfig:whole-block}
  \end{subfigure}
  \caption{Computation graphs for transformers and probabilistic transformers.}
  \label{fig:cmp-whole-block}
\end{figure*}

\section{Experiments}

We empirically compare probabilistic transformers with transformers on three tasks: masked language modeling, sequence labeling, and text classification. For each task, we use two different datasets. We also perform a syntactic test to evaluate the compositional generalization ability of our model.

\subsection{Tasks and Datasets}
\label{sec:task-dataset}

Here we briefly introduce our tasks and datasets. A detailed description is presented in Appendix~\ref{apx:task-dataset}.

\textbf{Masked Language Modeling (MLM)}. We perform MLM tasks on two corpora: the Penn TreeBank (PTB) \cite{marcinkiewicz1994building} and Brown Laboratory for Linguistic Information Processing (BLLIP) \cite{charniak2000bllip}. Following \citet{shen-etal-2022-unsupervised}, we randomly replace words with a mask token \texttt{<mask>} at a rate of 30\%. The performance of MLM is evaluated by measuring perplexity (lower is better) on masked words.

We project the final word representation of each mask token to the vocabulary. For transformers, we tie the projection parameters to the initial word embeddings. We find that this trick improves the performance of transformers.

\textbf{Sequence Labeling}. For sequence labeling tasks, we perform part-of-speech (POS) tagging on two datasets: the Penn TreeBank (PTB) \cite{marcinkiewicz1994building} and the Universal Dependencies (UD) \cite{de2021universal}.  We also perform named entity recognition (NER) on CoNLL-2003 \cite{tjong-kim-sang-de-meulder-2003-introduction}.

We directly project the final word representation of each word to the target tag set. For POS tagging, we evaluate the results by the accuracy of word-level predictions. For NER, we evaluate the results by measuring the F1 score of named entities.

\textbf{Text Classification}. We use the Stanford Sentiment Treebank (SST) \cite{socher-etal-2013-recursive} as the dataset. It has two variants: binary classification (SST-2) and fine-grained classification (SST-5).

For transformers, we add a \texttt{<CLS>} token at the front of the sentence and then project its representation to the tag set. For our model, we use the variant with a root node introduced in Section~\ref{sec:var-rn} and project the representation of the root node to the tag set.

\textbf{Syntactic Test}. To evaluate the compositional generalization abilities of our model, we perform a syntactic test on the COGS dataset \cite{kim-linzen-2020-cogs}. We follow the settings in \citet{ontanon2021making}, who cast the task as a sequence labeling task.

As in sequence labeling, we project word representations to tag sets. If all words in a sentence are correctly predicted, the sentence prediction will be counted as correct. We evaluate the results by the sentence-level accuracy of the predictions.

\begin{table*}[tb]
  \small
  \centering
  \begin{tabular}{@{}ccccc@{}}
  \toprule
  Task                 & Dataset & Metric                      & Transformer       & Probabilistic Transformer \\ \midrule
  \multirow{2}{*}{MLM} & PTB     & \multirow{2}{*}{Perplexity} & $58.43 \pm 0.58$  & $62.86 \pm 0.40$          \\
                       & BLLIP   &                             & $101.91 \pm 1.40$ & $123.18 \pm 1.50$         \\ \midrule
  \multirow{2}{*}{POS} & PTB     & \multirow{2}{*}{Accuracy}   & $96.44 \pm 0.04$  & $96.29 \pm 0.03$          \\
                       & UD      &                             & $91.17 \pm 0.11$  & $90.96 \pm 0.10$          \\ \midrule
  NER                  & CoNLL-2003 & F1                          & $74.02 \pm 1.11$  & $75.47 \pm 0.35$          \\ \midrule
  \multirow{2}{*}{CLS} & SST-2   & \multirow{2}{*}{Accuracy}   & $82.51 \pm 0.26$  & $82.04 \pm 0.88$         \\
                       & SST-5   &                             & $40.13 \pm 1.09$  & $42.77 \pm 1.18$         \\ \midrule
  Syntactic Test       & COGS    & Sentence-level Accuracy     & $82.05 \pm 2.18$  & $84.60 \pm 2.06$          \\ \bottomrule
  \end{tabular}
  \caption{Main results of probabilistic transformers compared with transformers.}
  \label{tab:main-result}
\end{table*}

\subsection{Settings}

We tune transformers and our model separately for each task except the syntactic test. For the syntactic test, we find that both transformers and our model easily reach 100\% accuracy on the validation set. This observation is consistent with \citet{ontanon2021making}. Therefore, instead of tuning, we use the best-performed setting of transformers in \citet{ontanon2021making} for our experiments. The hyperparameters of our model are determined by their counterparts of transformers based on the correspondence discussed in Section~\ref{sec:cp-to-trans}.

For our model, we integrate all the variants mentioned in Section~\ref{sec:variants} except the root node variant, which we only use for text classification tasks. We tune the tensor decomposition strategy on different tasks. For MLM tasks, we add a small L2 regularization term to the ternary scores in our model, which we experimentally find beneficial. We optimize both models using the Adam optimizer \cite{kingma2014adamam} with $\beta_1 = 0.9$, $\beta_2 = 0.999$.

\subsection{Results}

We report the average and standard deviation results of 5 random runs in Table~\ref{tab:main-result}. It shows that our model has a competitive performance compared with transformers. In most tasks, probabilistic transformers perform competitively to transformers. It is worth noting that in these experiments, probabilistic transformers have much fewer parameters than transformers. For most tasks, the number of parameters of our best model is about one-fifth to one-half of that of the best transformer.

We also conduct case studies of the dependency structures inferred by our model after training on downstream tasks. Similar to the case of self-attentions in transformers, the inferred dependency structures are only partially consistent with human intuition. See Appendix~\ref{apx:case-study} for details.

\section{Related Work}






There have been several studies trying to incorporate syntactic structures to transformers. \citet{strubell-etal-2018-linguistically} force one attention head to attend to predicted syntactic governors of input tokens. \citet{wang-etal-2019-tree, ahmad2021gate} try to integrate constituency or dependency structures into transformers. \citet{shen-etal-2021-structformer} propose a dependency-constrained self-attention mechanism to induce dependency and constituency structures. 
Our work deviates from all these previous studies in that we start from scratch with probabilistic modeling of word representations and dependencies, but obtain a model that is strikingly similar to transformers.

\section{Discussion}
\label{sec:discussion}

It is worth noting that in this work, our primary goal is not to propose and promote a new model to compete with transformers. Instead, it is our hope that our work could benefit the analysis and extension of transformers, as well as inspire future research of transformer-style models that are linguistically more principled, theoretically more well-founded, and empirically no less powerful than existing models. 
In the long run, we aim to bridge the gap between traditional statistical NLP and modern neural NLP, so that valuable ideas, techniques and insights developed over the past three decades in statistical NLP could find their place in modern NLP research and engineering.

The datasets used in our experiments have small to medium sizes (around 10k to 60k training sentences). Our preliminary experiments with MLM on larger data show that our models significantly underperform transformers, which suggests that our model may not be as scalable as transformers. One possible cause is the absence of a feed-forward structure in our model. Recent researches show that the feed-forward layers might serve as an important part of transformers \cite{pmlr-v139-dong21a}. Further research is needed to analyze this problem.

Our model can be extended in a few directions. Instead of discrete labels, we may assume $Z$ variables representing discrete vectors or even continuous vectors, which may lead to more complicated inference. We may model dependency labels by pairing every $H$ variable with a dependency label variable. While we focus on contextual word representation (i.e., encoding) in this paper, we may extend our probabilistic model to include a decoder. Considering the similarity between our model and transformers, we speculate that some of these extensions may be used to inspire extensions of transformers as well.

\section{Conclusion}

We present probabilistic transformers, a type of syntactic-aware probabilistic models for contextual word representation. A probabilistic transformer acquires discrete latent representations of all words in the input sentence by modeling a syntactic dependency structure of the input sentence. We use MFVI for approximate inference and find a striking resemblance between the computation graph of the inference procedure of our model and that of a transformer. Our experimental results demonstrate that our model performs competitively to transformers on small to medium sized datasets.

\section*{Limitations}

Though we have found a tight connection between probabilistic transformers and transformers in Section~\ref{sec:cp-to-trans}, this does not mean that our model can be directly used to interpret or modify transformers. For instance, in Section~\ref{sec:cmp-multi-head}, we find that $W^K$ and $W^V$ in transformers both correspond to $\mathbf{U}$ in probabilistic transformers. However, if we tie $W^K$ and $W^V$ in transformers, then we may observe a performance drop on some downstream tasks.

The performance of probabilistic transformers lags behind transformers on large datasets (>100k), which suggests that our model may not be as scalable as transformers. We have discussed this in Section~\ref{sec:discussion}. 

The way of positional encoding for probabilistic transformers leads to slower training and inference speed. On masked language modeling tasks, our model is about 3 times slower than transformers with either absolute or relative positional encoding, though it has much fewer parameters than transformers.

\section*{Acknowledgements}

This work was supported by the National Natural Science Foundation of China (61976139).

\bibliography{anthology,custom}
\bibliographystyle{acl_natbib}

\appendix



\section{Extended Entropic Frank-Wolfe}
\label{apx:ex-efw}

In Section~\ref{sec:var-rw}, we add message weights to the update function of the posterior marginal distributions. It follows an extension of the Entropic Frank-Wolfe algorithm \cite{le2021regularized}, which is a generalization of MFVI. Below we briefly introduce the algorithm and our extension following most of the notations in their paper.

\subsection{Entropic Frank-Wolfe}
\label{apx:ex-sub-efw}

Suppose we want to minimize a continuous differentiable energy function $E(\cdot)$. Vanilla Frank-Wolfe solves the problem $\min_{\mathbf{x} \in \mathcal{X}} E(\mathbf{x})$ by starting from a feasible $\mathbf{x}^{(0)} \in \mathcal{X}$ at time step 0, and iterating the following steps:
\begin{equation*}
  \begin{aligned}
    &\mathbf{p}^{(t)} \in \underset{\mathbf{p} \in \mathcal{X}}{\operatorname{argmin}}\left\langle\nabla E\left(\mathbf{x}^{(t)}\right), \mathbf{p}\right\rangle \\
    &\mathbf{x}^{(t+1)}=\mathbf{x}^{(t)}+\alpha_{t}\left(\mathbf{p}^{(t)}-\mathbf{x}^{(t)}\right)
  \end{aligned}
\end{equation*}
where $\alpha_{t} \in[0,1]$ follows some stepsize scheme, $\mathcal{X}$ is the value range of $\mathbf{x}$, and here we let $\mathbf{x} \in \mathbb{R}^{n \times d}$ be the concatenation of the distributions over the label set of all variables in CRF.

Regularized Frank-Wolfe \cite{le2021regularized} adds a regularization term $r(\cdot)$ to the objective. It solves the new objective $E(\mathbf{x}) + r(\mathbf{x})$ by iterating
\begin{equation*} \label{apxeq:reg-fw}
  \begin{aligned}
    &\mathbf{p}^{(t)} \in \underset{\mathbf{p} \in \mathcal{X}}{\operatorname{argmin}}\left\{\left\langle\nabla E\left(\mathbf{x}^{(t)}\right), \mathbf{p}\right\rangle+r(\mathbf{p})\right\} \\
    &\mathbf{x}^{(t+1)}=\mathbf{x}^{(t)}+\alpha_{t}\left(\mathbf{p}^{(t)}-\mathbf{x}^{(t)}\right)
  \end{aligned}
\end{equation*}
It has been proved that regularized Frank-Wolfe achieves a sublinear rate of convergence $O(1/\sqrt{t})$ for suitable stepsize schemes.

Entropic Frank-Wolfe is a special case of regularized Frank-Wolfe, which sets the regularization term as an entropy function $r(\mathbf{x}) = -\lambda H(\mathbf{x})$, where $H(\mathbf{x})=-\sum_{i \in \mathcal{V}} \sum_{s \in \mathcal{S}} x_{i s} \log x_{i s}$, $\mathcal{S}$ is the label set of the variables, $\mathcal{V}$ is the set of indices of the variables. Entropy Frank-Wolfe has a closed-form solution for the update process
\begin{equation} \label{apxeq:mfvi}
  \begin{aligned}
    \mathbf{p}^{(t)}&=\underset{\mathbf{p} \in \mathcal{X}}{\operatorname{argmin}}\left\{\left\langle\nabla E\left(\mathbf{x}^{(t)}\right), \mathbf{p}\right\rangle-\lambda H(\mathbf{p})\right\}\\
    &=\operatorname{softmax}\left(-\frac{1}{\lambda}\left(\nabla E\left(\mathbf{x}^{(t)}\right)\right)\right) \quad \forall t \geq 0
  \end{aligned}
\end{equation}
When $\lambda = 1$ and $\alpha_t = 1, \forall t \geq 0$, it is the same as the mean field algorithm.

\subsection{Extended Entropic Frank-Wolfe}

We extend the Entropic Frank-Wolfe algorithm by using a more general regularization term
\begin{equation*}
  r(\mathbf{x})=- \sum_{i \in \mathcal{V}} \lambda_i H(\mathbf{x}_i)
\end{equation*}
, where $\lambda_i>0$ is the regularization weight of the $i$-th variable and $H(\mathbf{x}_i)=-\sum_{s \in \mathcal{S}} x_{i s} \log x_{i s}$ is the entropy of $\mathbf{x}_i$ over the probability simplex $\Delta=\left\{\mathbf{x} \in \mathbb{R}^{d}: \mathbf{x} \geq \mathbf{0}, \mathbf{1}^{\top} \mathbf{x}=1\right\}$. It allows us to assign different regularization weights for different variables. We claim that the update function could be written as
\begin{equation} \label{apxeq:eefw}
  \begin{aligned}
    \mathbf{p}^{(t)}&=\underset{\mathbf{p} \in \mathcal{X}}{\operatorname{argmin}}\left\{\left\langle\nabla E\left(\mathbf{x}^{(t)}\right), \mathbf{p}\right\rangle-\lambda_i H(\mathbf{p}_i)\right\}\\
    &=\operatorname{softmax}\left(\mathbf{R}\right) \quad \forall t \geq 0
  \end{aligned}
\end{equation}
, where $\mathbf{R} \in \mathbb{R}^{nd}$ and
\begin{equation*}
\mathbf{R}_i = -\frac{1}{\lambda_i}\left(\nabla E\left(\mathbf{x}_i^{(t)}\right)\right) \quad \forall i \in \mathcal{V}
\end{equation*}

This extension is still a special case of the regularized Frank-Wolfe algorithm. As a result, it inherits all the convergence properties from the regularized Frank-Wolfe mentioned in the previous section. On the other hand, it is also an extension of MFVI, which allows adding a message weight to each variable during inference.

\subsection{A Proof for Extended Entropic Frank-Wolfe}

We give a simple proof to the close-form solution of extended Entropic Frank-Wolfe in Equation~\ref{apxeq:eefw}. Since the optimization could reduce to $n$ independent subproblems over each $i \in \mathcal{V}$, We only need to give the closed-form solution to each subproblem:

\begin{lemma}
  For a given vector $\mathbf{c} \in \mathbb{R}^{d}$, $\lambda > 0$, the optimal solution $\mathrm{z}^{*}$ to 
  \begin{equation*}
    \min _{\mathbf{z} \in \Delta}\left\{\langle\mathbf{c}, \mathbf{z}\rangle+\lambda \sum_{s=1}^{d} z_{s} \log z_{s}\right\}
  \end{equation*}
  is $\mathbf{z}^{*}=\operatorname{softmax}(-\frac{1}{\lambda}\mathbf{c})$, where $\Delta$ is the probability simplex $\left\{\mathbf{x} \in \mathbb{R}^{d}: \mathbf{x} \geq \mathbf{0}, \mathbf{1}^{\top} \mathbf{x}=1\right\}$.
\end{lemma}

\begin{proof}
  We can rewrite the problem as
  \begin{equation*}
    \begin{aligned}
    &\min_{\mathbf{z}} \quad \langle\mathbf{c}, \mathbf{z}\rangle+\lambda \sum_{s=1}^{d} z_{s} \log z_{s}\\
    &\begin{array}{r@{\quad}r@{\quad}l@{\quad}l}
    s.t. &\mathbf{1}^{\top} \mathbf{z} & =1, \\
        &-\mathbf{z} & \leq \mathbf{0}, \\
    \end{array}
    \end{aligned}
  \end{equation*}
  The Lagrangian of the above problem is given by
  \begin{equation*}
    \begin{aligned}
    L(\mathbf{z}, \boldsymbol{\mu}, \nu) & =\langle\mathbf{c}, \mathbf{z}\rangle+\lambda\sum_{s=1}^{d} z_{s} \log z_{s}\\
    &+\boldsymbol{\mu}^{\top}(-\mathbf{z})+\nu\left(\mathbf{1}^{\top} \mathbf{z}-1\right) \\
    & =-\nu+\sum_{s=1}^{d}(c_{s} z_{s}+\lambda z_{s} \log z_{s}\\
    &-\mu_{s} z_{s}+\nu z_{s})
    \end{aligned}
  \end{equation*}
  where $\boldsymbol{\mu}=\left(\mu_{1}, \mu_{2}, \ldots, \mu_{d}\right) \geq \mathbf{0}$ and $\nu \in \mathbb{R}$ are the Lagrange multipliers.

  Since the given problem is convex and there exists $\mathbf{z} \in \mathbb{R}^{d}$ such that $\mathbf{1}^{\top} \mathbf{z}=1$ and $\mathbf{z}>\mathbf{0}$, the Slater's constraint qualification holds. Thus, it suffices to solve the following Karush-Kuhn-Tucker (KKT) system to obtain the optimal solution:
  \begin{equation*}
  \begin{aligned}
  c_{s}+\lambda\log z_{s}+1-\mu_{s}+\nu & =0 \quad \forall 1 \leq s \leq d, \\
  \mathbf{1}^{\top} \mathbf{z} & =1, \\
  \mathbf{z} & \geq \mathbf{0}, \\
  \boldsymbol{\mu} & \geq \mathbf{0}, \\
  \mu_{s} z_{s} & =0 \quad \forall 1 \leq s \leq d .
  \end{aligned}
  \end{equation*}
  The first equation implies $\forall 1 \leq s \leq d, z_{s}>0$, and thus in combination with the last, we obtain $\forall 1 \leq s \leq d, \mu_{s}=0$. Therefore, the first equation becomes
  \begin{equation*}
    c_{s}+\lambda\log z_{s}+1+\nu = 0
  \end{equation*}
  $\forall 1 \leq s \leq d$. Rewrite the equation as
  \begin{equation*}
  z_{s}=\exp \left(\frac{-1-\nu}{\lambda}\right) \exp \left(-\frac{1}{\lambda}c_{s}\right)
  \end{equation*}
  $\forall 1 \leq s \leq d$. Summing up this result for all $s$, and taking into account the second equation, we have
  \begin{equation*}
    \sum_{s=1}^d \exp \left(\frac{-1-\nu}{\lambda}\right) \exp \left(-\frac{1}{\lambda}c_{s}\right) = 1
  \end{equation*}
  That is,
  \begin{equation*}
    \exp \left(\frac{-1-\nu}{\lambda}\right)=\frac{1}{\sum_{s=1}^{d} \exp \left(-\frac{1}{\lambda}c_{s}\right)}
  \end{equation*}
  Combine these two formulas, we have
  \begin{equation*}
    z_{s}=\frac{\exp \left(-\frac{1}{\lambda}c_{s}\right)}{\sum_{t=1}^{d} \exp \left(-\frac{1}{\lambda}c_{t}\right)}
  \end{equation*}
  $\forall 1 \leq s \leq d$. In other words, $\mathbf{z}=\operatorname{softmax}(-\frac{1}{\lambda}\mathbf{c})$.
\end{proof}

\subsection{Inference in CRF}

In this work, we apply the extended Entropic Frank-Wolfe to do inference in the CRF. Let $\mathbf{s} = (Z_1, \cdots, Z_n, H_1^{(1)}, \cdots, H_n^{(1)}, H_1^{(2)}, \cdots, H_n^{(h)})$ denote an assignment to all the random variables. Our CRF encodes the joint distribution
\begin{equation*}
    p(\mathbf{s}) = \frac{1}{Z} \prod_i \phi_u(Z_i) \prod_c \prod_i \prod_{j\ne i} \phi_t(H_i^{(c)}, Z_i, Z_j)
\end{equation*}
where $Z$ is a normalization factor. The objective is to find an assignment $\mathbf{s}$ that maximizes the joint distribution $p(\mathbf{s})$. To express in the form of an energy function, let $p(\mathbf{s}) = \frac{1}{Z} \exp(-e(\mathbf{s}))$, we have
\begin{equation*}
    e(\mathbf{s}) = - \sum_i \mathbf{S}_{w_i, Z_i} - \sum_c \sum_i \sum_{j\ne i} \mathbbm{1}_{H_i = j} \mathbf{T}^{(c)}_{Z_i, Z_j}
\end{equation*}
where $\mathbbm{1}_{H_i = j}$ is an indicator function, which is equal to 1 if $H_i = j$ and is equal to 0 otherwise. The objective could now be expressed as minimizing the energy function $e(\mathbf{s})$.

In general, the problem of CRF inference is NP-Hard \cite{shimony1994finding}. In MFVI, we solve the continuous relaxation of the CRF problem instead. Let $\mathcal{X}$ be the simplex. That is, we allow a marginal distribution for each random variable. As in Section~\ref{sec:inference}, let $Q_i(\cdot)$ be the approximate marginal distribution over $Z_i$ and $Q_{ic}(\cdot)$ be the approximate marginal distribution over $H_i^{(c)}$. The energy function is then
\begin{equation*}
\begin{split}
    &E(Q_*) = - \sum_i \sum_a Q_i(a) \mathbf{S}_{w_i, a}\\
    &- \sum_c \sum_i \sum_{j\ne i} \sum_a \sum_b Q_i(a) Q_j(b) Q_{ic}(j) \mathbf{T}^{(c)}_{a, b}  
\end{split}
\end{equation*}
Then we have
\begin{equation*}
\begin{split}
    \frac{\partial E}{\partial Q_i(a)} &= - \mathbf{S}_{w_i, a} - \sum_c \sum_{j\ne i} \sum_b\\
    & \left( Q_j(b) Q_{ic}(j) \mathbf{T}^{(c)}_{a, b} + Q_j(b) Q_{jc}(i) \mathbf{T}^{(c)}_{b, a} \right) \\
    \frac{\partial E}{\partial Q_{ic}(j)} &= - \sum_a \sum_b Q_i(a) Q_j(b) \mathbf{T}^{(c)}_{a, b}
\end{split}
\end{equation*}
In MFVI, the update for each distribution is the softmax of the derivative (let $\lambda = 1$ and $\alpha_t = 1, \forall t \geq 0$ in Equation~\ref{apxeq:mfvi}). That is,
\begin{align*}
    Q_i^{(t)}(a) &\propto \exp \left( - \frac{\partial E^{(t-1)}}{\partial Q_i^{(t-1)}(a)} \right) \\
    Q_{ic}^{(t)}(j) &\propto \exp \left( - \frac{\partial E^{(t-1)}}{\partial Q_{ic}^{(t-1)}(j)} \right)
\end{align*}
Together with Equation~\ref{eq:msg-f} and \ref{eq:msg-g}, we have 
\begin{align*}
    \frac{\partial E^{(t-1)}}{\partial Q_i^{(t-1)}(a)} &= - \mathbf{S}_{w_i, a} - \mathcal{G}_{i}^{(t-1)}(a) \\
    \frac{\partial E^{(t-1)}}{\partial Q_{ic}^{(t-1)}(j)} &= - \mathcal{F}_{ic}^{(t-1)}(j)
\end{align*}
, which directly leads us to Formula~\ref{eq:msg-qz} and \ref{eq:msg-qh}.

In the extended Entropic Frank-Wolfe, the update for each distribution is the regularized softmax of the derivative (Equation~\ref{apxeq:eefw}). That is,
\begin{align*}
    Q_i^{(t)}(a) &\propto \exp \left( - \frac{1}{\lambda_{i}} \frac{\partial E^{(t-1)}}{\partial Q_i^{(t-1)}(a)} \right) \\
    Q_{ic}^{(t)}(j) &\propto \exp \left( - \frac{1}{\lambda_{ic}} \frac{\partial E^{(t-1)}}{\partial Q_{ic}^{(t-1)}(j)} \right)
\end{align*}
Let $\lambda_{i} = \lambda_{Z} > 0, \lambda_{ic} = \lambda_{H} > 0$, $\forall i, c$. Then it is equivalent to Formula~\ref{eq:reg-weight-z} and \ref{eq:reg-weight-h} with regularization weight $\lambda_Z > 0$ for $Z$ variables and $\lambda_H > 0$ for $H$ variables.

\subsection{The Choice of Message Weights}
\label{apx:choice}

In Section~\ref{sec:var-rw}, we set $\lambda_Z = 1$ and $\lambda_H = \frac{1}{d}$ by default. This choice comes from a theoretical analysis similar to \citet{vaswani2017attention}, and we empirically find it helpful to improve the performance.

Assume that the ternary scores in $\mathbf{T}$ are independent random variables with mean 0 and variance $\sigma^2$. Then from Equation~\ref{eq:msg-f}, we know that $\mathcal{F}_{ic}^{(t)}(j)$ is a weighted sum of these random variables. Suppose the weights are uniformly distributed, then $\mathcal{F}_{ic}^{(t)}(j)$ has mean 0 and variance $\frac{d^2}{(d^2)^2}\sigma^2 = \frac{1}{d^2} \sigma^2$. Since $d$ is usually set to several hundred, this might result in a small variance in the message received by $H$ variables and thus lead to uniformly distributed $H$ variables. To balance this effect, we set $\lambda_H = \frac{1}{d}$ such that the variance of $\frac{1}{\lambda_H} \mathcal{F}_{ic}^{(t)}(j)$ is still $\sigma^2$. From Equation~\ref{eq:msg-g} we know that the variance of $\mathcal{G}_{i}^{(t)}(a)$ is $\frac{2(n-1)}{hd} \sigma^2$. Here, since $n$ varies in sentences, it is impossible to set a fixed $\lambda_Z$ that always recovers the original variance $\sigma^2$. Compared to $\mathcal{F}_{ic}^{(t)}(j)$, the variance of $\mathcal{G}_{i}^{(t)}(a)$ does not change significantly. For simplicity, we set $\lambda_Z = 1$.

\section{More Extensions and Variants}
\label{apx:more-variants}

We have introduced several extensions and variants that are beneficial to the model performance in Section~\ref{sec:variants}. There are some other variants that we find do not bring significant improvement empirically, but might also be meaningful and have interesting correspondences to transformers.

\subsection{Step Size}
\label{apx:step-size}

In our model, we can retain information between iterations and do partially update with a proper step size. Let
\begin{align*}
  Q_i^{\star (t)}(a) &\propto \exp\left(\mathbf{S}_{w_i, a} + \mathcal{G}_{i}^{(t-1)}(a)\right) \\
  Q_{ic}^{\star (t)}(j) &\propto \exp\left(\mathcal{F}_{ic}^{(t-1)}(j)\right)
\end{align*}
be the original posterior marginal distributions of the variables at time step $t$, which is the same as Formula~\ref{eq:msg-qz} and \ref{eq:msg-qh}. We have the posterior distributions with step size
\begin{align*}
  Q_i^{(t)}(Z_i) &= \alpha_Z  Q_i^{\star (t)}(Z_i) + (1-\alpha_Z)  Q_i^{(t-1)}(Z_i) \\
  Q_{ic}^{(t)}(H_i^{(c)}) &= \alpha_H  Q_{ic}^{\star (t)}H_i^{(c)} + (1-\alpha_H)  Q_{ic}^{(t-1)}H_i^{(c)}
\end{align*}
where $\alpha_Z, \alpha_H \in (0, 1]$ are the step sizes of each update. When $\alpha_Z = \alpha_H = 1$, it is equivalent to the original model. We initialize these distribution by Formula~\ref{eq:init-qz} and \ref{eq:init-qh}. 

\subsection{Damping}

Similar to step size in Appendix~\ref{apx:step-size}, the damping approach also aims at retaining information between iterations. Instead of partially updating the posterior distribution, the damping approach partially updates the messages.

We define messages in time step $t$ as
\begin{align}
  \label{eq:apx-msg-z} M_i^{(t)}(a) &= \mathbf{S}_{w_i, a} + \mathcal{G}_{i}^{(t-1)}(a) \\
  \label{eq:apx-msg-h} M_{ic}^{(t)}(j) &= \mathcal{F}_{ic}^{(t-1)}(j)
\end{align}
where $M_i^{(t)}(Z_i)$ is the message passed to $Z_i$ and $M_{ic}^{(t)}(H_i^{(c)})$ is the message passed to $H_i^{(c)}$. Thus, Formula~\ref{eq:msg-qz} and \ref{eq:msg-qh} can be written as
\begin{align*}
  Q_i^{(t)}(a) &\propto \exp\left(M_i^{(t)}(a)\right) \\
  Q_{ic}^{(t)}(j) &\propto \exp\left(M_{ic}^{(t)}(j)\right)
\end{align*}

Now, we add damping factors $\beta_Z$ and $\beta_H$, which restrict the message update between iterations. We change Equation~\ref{eq:apx-msg-z} and \ref{eq:apx-msg-h} to
\begin{align*}
  &\begin{aligned}
    M_i^{(t)}(a) = &(1-\beta_Z) \left( \mathbf{S}_{w_i, a} + \mathcal{G}_{i}^{(t-1)}(a) \right) \\
    & + \beta_Z M_i^{(t-1)}(a)
  \end{aligned} \\
  &M_{ic}^{(t)}(j) = (1-\beta_H) \left( \mathcal{F}_{ic}^{(t-1)}(j) \right) + \beta_H M_{ic}^{(t-1)}(j)
\end{align*}

We initialize the message by
\begin{align*}
  M_i^{(0)}(a) &= \mathbf{S}_{w_i, a} \\
  M_{ic}^{(0)}(j) &= 0
\end{align*}

When $\beta_Z = \beta_H = 0$, there is no damping in the update process and it is equivalent to the original model. When $\beta_Z = 0.5$ and $\beta_H = 0$, it is similar to the residual connection in transformers. When $\beta_Z = \beta_H = 0.5$, it is similar to the residual attention mechanism proposed in RealFormer \cite{he-etal-2021-realformer}.

\subsection{Global Variables}
\label{apx:global-vars}

As we mentioned in Section~\ref{sec:cmp-whole-block}, probabilistic transformers do not have a feed-forward structure as in transformers. Feed-forward layers, however, constitute two-thirds of a transformer model's parameters. Recent researches show that the feed-forward layers might serve as an important part of transformers \cite{pmlr-v139-dong21a, geva-etal-2021-transformer, geva2022transformer}.

Inspired by \citet{sukhbaatar2019augmenting}, who combines the feed-forward layer and the self-attention layer into a unified all-attention layer, we design a similar structure based on dependency relations. Intuitively, we could add some global variables that are similar to the latent word representations ($Z$ variables) but these representations are global features that do not change with input sentences. We will introduce 3 different model designs below.


\subsubsection{All-dep}


Based on the intuition above, we add some global variables to the CRF model. Define $F_i$ as the $i$-th discrete global feature variable with the same label set as $Z$ variables, representing the global features of the corpus. The total number of global feature variables is $m$. These variables are observed and the distributions on the label set will not change during inference. The head of each word could either be another word or a global feature variable. That is, $H_i^{(c)} \in \{1, 2, \cdots, n, n+1, \cdots, n+m\}$. 


Then, for each word $w_i$ and global feature $F_j$ in channel $c$, we define a ternary potential function over $Z_i$, $H_i^{(c)}$ and $F_j$, which evaluates the compatibility between the labels of the word and the global feature of the entire corpus.
\begin{equation*}
  \begin{aligned}
    \phi_t(&H_i^{(c)}, Z_i, F_j)=\\
    &\left\{
      \begin{aligned}
      \exp(\mathbf{T}^{''(c)}_{Z_i, F_j}) & , & H_i^{(c)}=n+j \\
      1 & , & \text{otherwise}
      \end{aligned}
    \right.
  \end{aligned}
\end{equation*}
where $\mathbf{T}^{''(c)} \in \mathbb{R}^{d \times d}$ is a score matrix for channel $c$.

An illustration of the CRF model is shown in Figure~\ref{fig:factor-graph-naive-all-dep}. We call this setting \textit{all-dep} since the head of each word could either be another word or a dummy global feature variable. It follows the \textit{all-attn} setting in \citet{sukhbaatar2019augmenting}.

\begin{figure*}
  \centering
  \includegraphics[page=3,width=.8\textwidth,trim=20 200 200 60,clip]{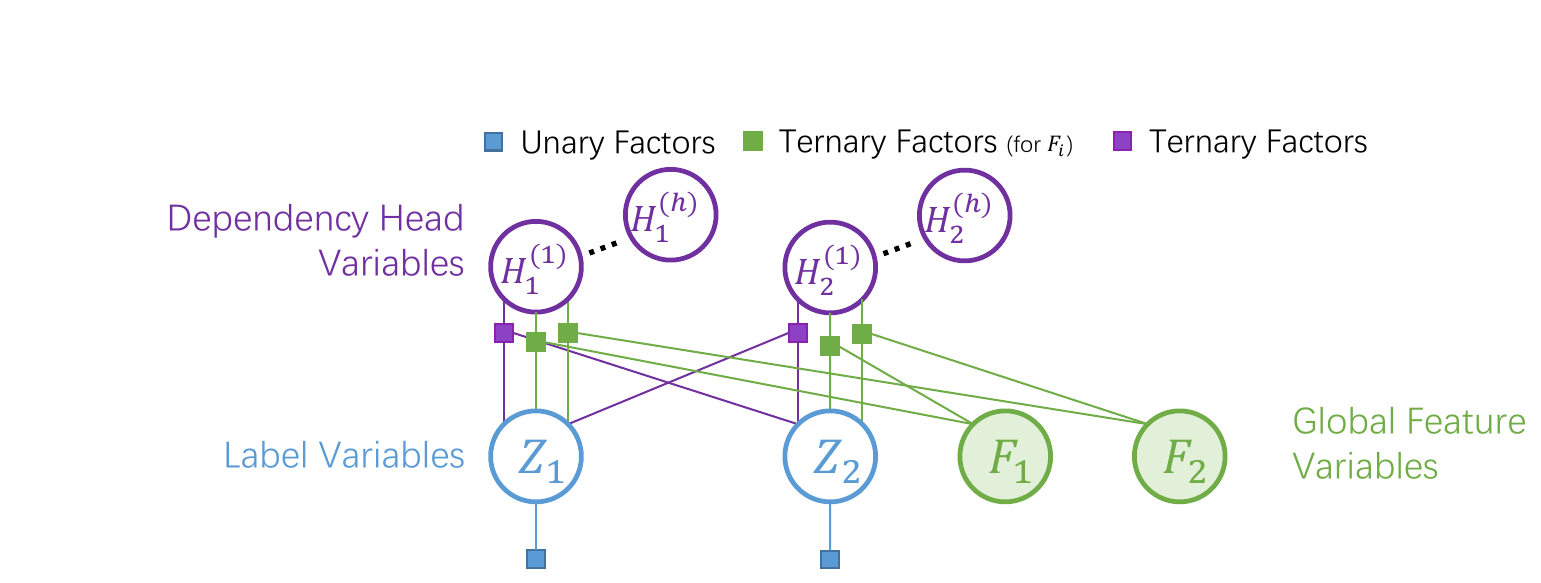}
  \caption{The factor graph for an intuitive design of CRF model with global variables where $n=m=2$. For clarity, ternary factors that connect to $H_i^{(c)}$ with $c>1$ are not shown in the figure.}
  \label{fig:factor-graph-naive-all-dep}
\end{figure*}

Notice that $F_j$ is a variable that does not participate in inference. It could be seen as part of the model. Thus, we could design an equivalent model that does not contain global feature variables but have a binary factor between $Z_i$ and $H_i^{(c)}$:
\begin{equation*}
  \begin{aligned}
    &\phi_b(H_i^{(c)}, Z_i)=\\
    &\left\{
    \begin{aligned}
    \prod_g \exp(P(F_{H_i^{(c)}-n} = g) \mathbf{T}^{''(c)}_{Z_i, g}) & , & H_i^{(c)}>n \\
    1 & , & \text{otherwise}
    \end{aligned}
    \right.
  \end{aligned}
\end{equation*}
where $P(F_{i} = g)$ is the probability that the $i$-th global variable has label $g$. It can be proved that the MFVI inference process for the model with global feature variables and the model with binary factors is the same. Move the product inside the exponential term, we have
\begin{equation*}
  \begin{aligned}
    &\phi_b(H_i^{(c)}, Z_i)=\\
    &\left\{
    \begin{aligned}
    \exp( \sum_g P(F_{H_i^{(c)}-n} = g) \mathbf{T}^{''(c)}_{Z_i, g}) & , & H_i^{(c)}>n \\
    1 & , & \text{otherwise}
    \end{aligned}
    \right.
  \end{aligned}
\end{equation*}
The term inside the exponential is a weighted sum of ternary scores. We may re-formulate this potential function with a simplified term:
\begin{equation*}
  \begin{aligned}
    \phi_b(&H_i^{(c)}, Z_i)=\\
    &\left\{
    \begin{aligned}
    \exp(\mathbf{B}_{H_i^{(c)}-n,Z_i}^{(c)}) & , & H_i^{(c)}>n \\
    1 & , & \text{otherwise}
    \end{aligned}
    \right.
  \end{aligned}
\end{equation*}
where $\mathbf{B}^{(c)} \in \mathbb{R}^{m, d}$ is a score matrix for channel $c$. The weighted sum of ternary scores could be regarded as a neural parameterization of the binary scores $\mathbf{B}^{(c)}$. An illustration of the simplified CRF model is shown in Figure~\ref{fig:factor-graph-all-dep}.

\begin{figure*}
  \centering
  \includegraphics[page=4,width=.7\textwidth,trim=50 200 300 110,clip]{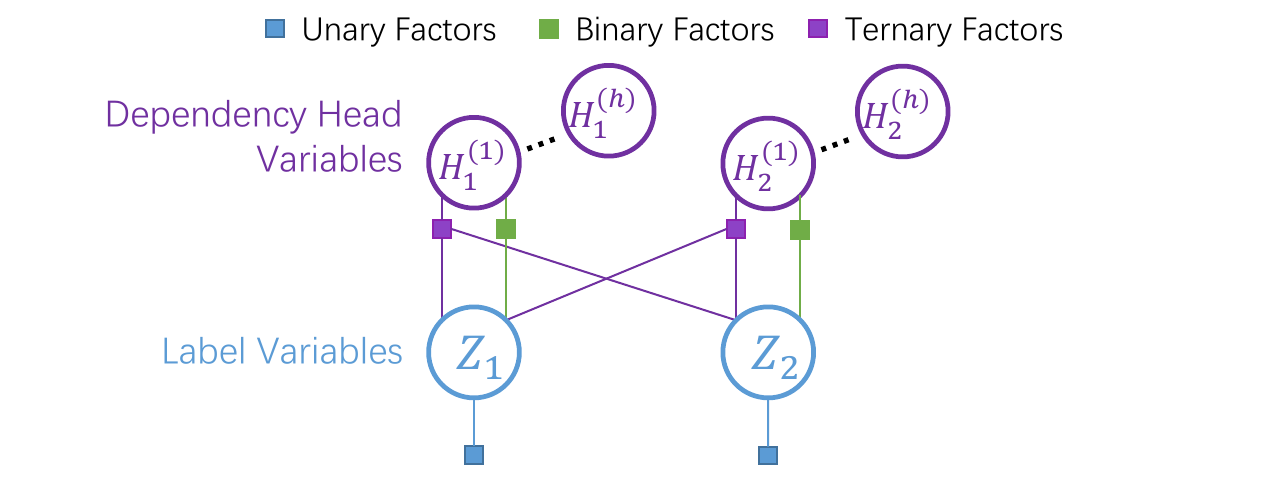}
  \caption{An equivalent factor graph for the \textit{all-dep} CRF model in Figure~\ref{fig:factor-graph-naive-all-dep}.}
  \label{fig:factor-graph-all-dep}
\end{figure*}


Given the model above, we can now derive the following iterative update equations of posterior distribution:
\begin{equation} \label{apxeq:msg-f-all-dep}
\begin{aligned}
  &\mathcal{F}_{ic}^{(t)}(j)= \\
  &\left\{
  \begin{aligned}
  \sum_{a}\sum_{b}\left(Q_i^{(t)}(a)Q_j^{(t)}(b)\mathbf{T}_{a,b}^{(c)} \right)  & , & j \le n \\
  \sum_{a}\left(Q_i^{(t)}(a)\mathbf{B}^{(c)}_{j,a}\right) & , & j > n
  \end{aligned}
  \right.  
\end{aligned}
\end{equation}
\begin{equation} \label{apxeq:msg-g-all-dep}
  \begin{aligned}
    \mathcal{G}_{i}^{(t)}(a)=&\sum_c\sum_{j \ne i, j \le n}\sum_{b}\left(Q_{ic}^{(t)}(j)Q_j^{(t)}(b)\mathbf{T}_{a,b}^{(c)}\right.\\
    &\left. + Q_{jc}^{(t)}(i)Q_j^{(t)}(b)\mathbf{T}_{b,a}^{(c)}\right) \\
    & + \sum_c\sum_{j > n} Q_{ic}^{(t)}(j)\mathbf{B}^{(c)}_{j,a}
  \end{aligned}
\end{equation}
where
\begin{align}
  Q_i^{(t)}(a)&\propto \exp\left(\mathbf{S}_{w_i, a} + \mathcal{G}_{i}^{(t-1)}(a)\right) \\
  \label{apxeq:msg-qh-all-dep} Q_{ic}^{(t)}(j)&\propto \exp\left(\mathcal{F}_{ic}^{(t-1)}(j)\right)
\end{align}

The initialization of the posterior marginal distributions $Q_i^{(t)}(\cdot)$ and $Q_{ic}^{(t)}(\cdot)$ is the same as Formula~\ref{eq:init-qz} and \ref{eq:init-qh}. Notice that $F_{ic}^{(t)} \in \mathbb{R}^{n+m}$ looks like a concatenation of a context vector and a persistent vector in all-attention networks \cite{sukhbaatar2019augmenting}.

\subsubsection{Dep-split}

Following the \textit{attn-split} setting in \citet{sukhbaatar2019augmenting}, we also design a \textit{dep-split} version of our model. In each channel, we split the head of each word into two heads: one for the head word in the sentence and one for the global feature. We call the heads for global features `global heads'.

Denote $G_i^{(c)} \in \{1, \cdot, m\}$ as the global head variable for $i$-th word in channel $c$. $H_i^{(c)} \in \{1, \cdot, n\}$ is still the variable representing the syntactic dependency head of the $i$-th word in the $c$-th channel. Similar to the approaches in the \textit{all-dep} setting, we define a simplified binary potential function for $Z_i$ and $G_i^{(c)}$
\begin{equation}
  \phi_b(G^{(c)}_i=k, Z_i=a)=\exp\left(\mathbf{B}_{k,a}^{(c)}\right)
\end{equation}
Figure~\ref{fig:factor-graph-dep-split} illustrates the CRF model of the \textit{dep-split} setting.


\begin{figure*}[t]
  \centering
  \includegraphics[page=5,width=.7\textwidth,trim=50 200 280 110,clip]{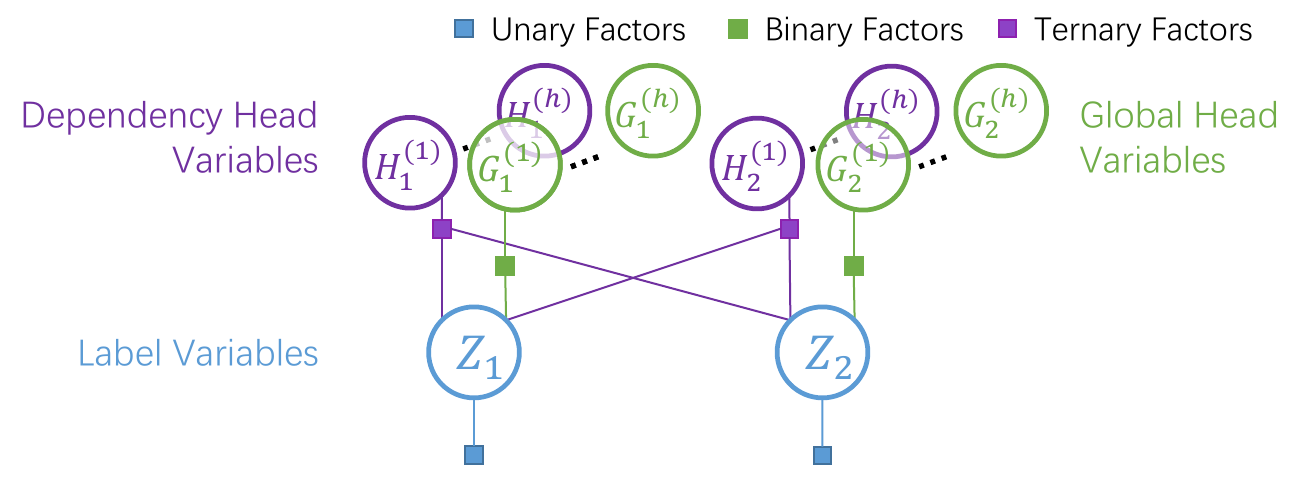}
  \caption{The factor graph for the \textit{dep-split} CRF model where $n=2$. For clarity, binary and ternary factors with channel $c>1$ are not shown in the figure.}
  \label{fig:factor-graph-dep-split}
\end{figure*}

We could derive the following iterative update equations of posterior distribution:
\begin{align}
  &\mathcal{F}_{ic}^{(t)}(j)=\sum_{a}\sum_{b}\left(Q_i^{(t)}(a)Q_j^{(t)}(b)\mathbf{T}_{a,b}^{(c)} \right) \\
  &\mathcal{H}_{i,k,c}^{(t)}=\sum_{a}\left(Q_i^{(t)}(a)\mathbf{B}^{(c)}_{k,a}\right) \\
  &\begin{aligned}
    \mathcal{G}_{i}^{(t)}(a)&=\sum_c\sum_{j \ne i}\sum_{b}Q_{ic}^{(t)}(j)Q_j^{(t)}(b)\mathbf{T}_{a,b}^{(c)} \\
    & + \sum_c\sum_{j \ne i}\sum_{b} Q_{jc}^{(t)}(i)Q_j^{(t)}(b)\mathbf{T}_{b,a}^{(c)} \\
    & + \sum_c\sum_k Q_{ic}^{'(t)}(k)\mathbf{B}^{(c)}_{k,a}
  \end{aligned}
\end{align}
where
\begin{align}
  Q_i^{(t)}(a)&\propto \exp\left(\mathbf{S}_{w_i, a} + \mathcal{G}_{i}^{(t-1)}(a)\right) \\
  Q_{ic}^{(t)}(j)&\propto \exp\left(\mathcal{F}_{ic}^{(t-1)}(j)\right) \\
  Q_{ic}^{'(t)}(k)&\propto \exp\left(\mathcal{H}_{i,k,c}^{(t-1)}\right)
\end{align}
are the approximate marginal distributions at time step $t$, with $Q_{ic}^{'(t)}(\cdot)$ over $G_i^{(c)}$. We initialize these distributions by Formula~\ref{eq:init-qz}, \ref{eq:init-qh} and
\begin{equation}
  Q_{ic}^{'(0)}(k) \propto 1
\end{equation}

\subsubsection{Single-split}

Following the \textit{single-split} setting in \citet{sukhbaatar2019augmenting}, we design a CRF model that is similar to the \textit{dep-split} model but only allows one global head for each word. We also call this setting \textit{single-split}. Denote $G_i$ as the global head variable for $i$-th word with a label set of size $m$. We define a binary potential for $Z_i$ and $G_i$
\begin{equation}
  \phi_b(G_i=k, Z_i=a)=\exp\left(\mathbf{B}_{k,a}\right)
\end{equation}
where $\mathbf{B} \in \mathbb{R}^{m \times d}$ is a score matrix. Figure~\ref{fig:factor-graph-single-split} illustrates the CRF model of the \textit{single-split} setting.

\begin{figure*}
  \centering
  \includegraphics[page=6,width=.67\textwidth,trim=50 200 320 110,clip]{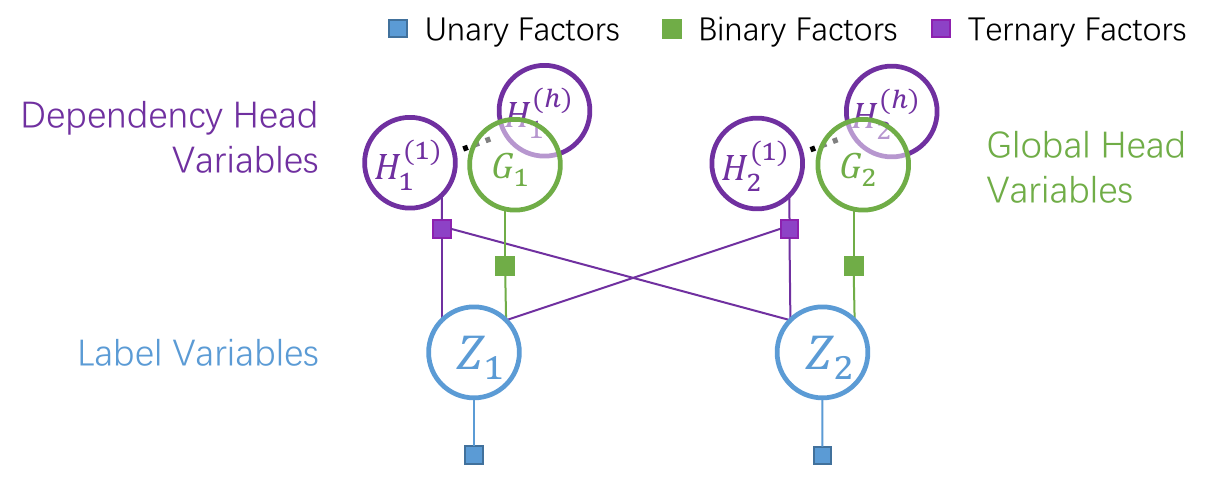}
  \caption{The factor graph for the \textit{single-split} CRF model where $n=2$. For clarity, ternary factors with channel $c>1$ are not shown in the figure.}
  \label{fig:factor-graph-single-split}
\end{figure*}

We could derive the following iterative update equations of posterior distribution:
\begin{align}
  &\mathcal{F}_{ic}^{(t)}(j)=\sum_{a}\sum_{b}\left(Q_i^{(t)}(a)Q_j^{(t)}(b)\mathbf{T}_{a,b}^{(c)} \right) \\
  &\mathcal{H}_{i,k}^{(t)}=\sum_{a}\left(Q_i^{(t)}(a)\mathbf{B}_{k,a}\right) \\
  &\begin{aligned}
    \mathcal{G}_{i}^{(t)}(a)&=\sum_c\sum_{j \ne i}\sum_{b}Q_{ic}^{(t)}(j)Q_j^{(t)}(b)\mathbf{T}_{a,b}^{(c)} \\
    & + \sum_c\sum_{j \ne i}\sum_{b} Q_{jc}^{(t)}(i)Q_j^{(t)}(b)\mathbf{T}_{b,a}^{(c)} \\
    & + \sum_k Q_{i}^{'(t)}(k)\mathbf{B}_{k,a}
  \end{aligned}
\end{align}
where
\begin{align}
  Q_i^{(t)}(a)&\propto \exp\left(\mathbf{S}_{w_i, a} + \mathcal{G}_{i}^{(t-1)}(a)\right) \\
  Q_{ic}^{(t)}(j)&\propto \exp\left(\mathcal{F}_{ic}^{(t-1)}(j)\right) \\
  Q_{i}^{'(t)}(k)&\propto \exp\left(\mathcal{H}_{i,k}^{(t-1)}\right)
\end{align}
are the approximate marginal distributions at time step $t$, with $Q_{i}^{'(t)}(\cdot)$ over $G_i$. We initialize these distributions by Formula~\ref{eq:init-qz}, \ref{eq:init-qh} and
\begin{equation}
  Q_{i}^{'(0)}(k) \propto 1
\end{equation}

\textit{single-split} might be the setting that has the most similar computation process to that of transformers. If we consider the tensorized form of \textit{single-split}, then for the posterior distributions of all the $G$ variables $Q_g^{(t)} \in \mathbb{R}^{n \times m}$, we have
\begin{align}
  &\mathcal{F}_{c}^{(t)} = Q_z^{(t)} \mathbf{T}^{(c)} Q_z^{(t)T} \\
  &\mathcal{H}^{(t)} = Q_z^{(t)} \mathbf{B}^T \\
  &\begin{aligned}
    \mathcal{G}^{(t)} =& \sum_c Q_{h,c}^{(t)} Q_z^{(t)} \mathbf{T}^{(c)T}\\
    &+ \sum_c Q_{h,c}^{(t)T} Q_z^{(t)} \mathbf{T}^{(c)}\\
    &+ Q_g^{(t)} \mathbf{B}
  \end{aligned}
\end{align}
where
\begin{align}
  Q_z^{(t)} &= \sigma\left(\mathbf{S} + \mathcal{G}^{(t-1)}\right) \\
  \label{apxeq:msg-tensor-ss-qh-novar} Q_{h,c}^{(t)} &= \sigma\left(\mathcal{F}_{c}^{(t-1)}\right) \\
  \label{apxeq:msg-tensor-ss-qg-novar} Q_{g}^{(t)} &= \sigma\left(\mathcal{H}^{(t-1)}\right)
\end{align}
With the similar trick in Section~\ref{sec:cp-to-trans}, we have
\begin{equation}
\begin{aligned}
  Q_z^{(t)} =& \sigma(\mathbf{S} + 2 \sum_c \operatorname{channel}_c \mathbf{U}^{(c)T} \\
  &+ \operatorname{GFU}(Q_z^{(t-1)}))    
\end{aligned}
\end{equation}
where 
\begin{align}
  \operatorname{channel}_c &= \sigma\left(\frac{Q_c K_c^T}{\lambda_H}\right) V_c \\
  \operatorname{GFU}(x) &= \sigma\left(x \mathbf{B}^T\right) \mathbf{B}
\end{align}
where we can regard $\operatorname{GFU}$ as an operator that updates the latent word representations from global features. An illustration of the computation process is shown in Figure~\ref{apxfig:single-split}. From Figure~\ref{apxfig:cmp-single-split-ffn}, we can see that the feed-forward structure in transformers is very similar to the global feature update process in probabilistic transformers with global variables.

\begin{figure}[t]
  \centering
  \includegraphics[page=7,width=.47\textwidth,trim=260 90 255 100,clip]{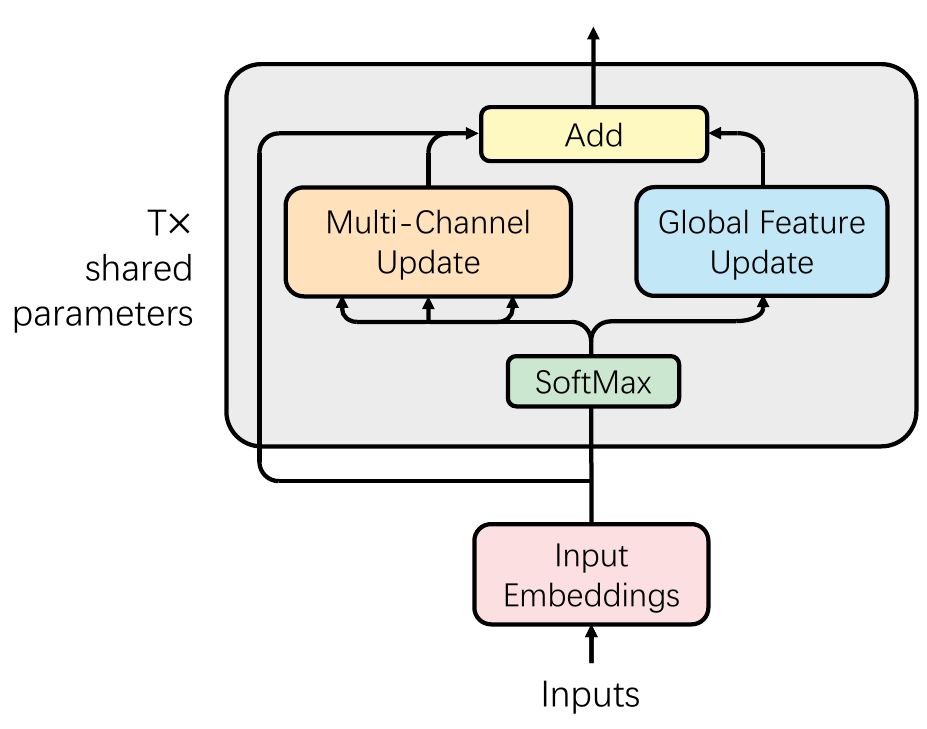}
  \caption{Computation graph for the \textit{single-split} probabilistic transformer.}
  \label{apxfig:single-split}
\end{figure}

\begin{figure}
  \centering
  \includegraphics[page=8,width=.47\textwidth,trim=310 190 200 110,clip]{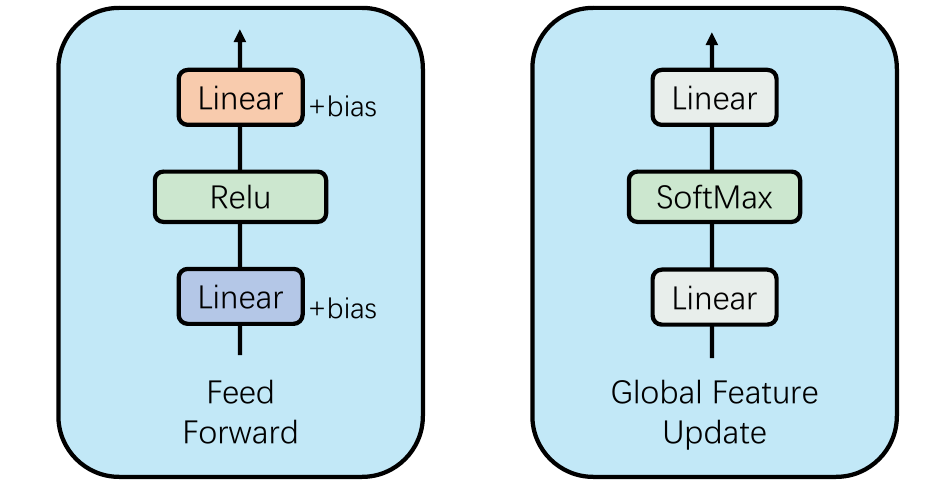}
  \caption{Computation graph for feed-forward in transformers and global feature update in probabilistic transformers with global variables.}
  \label{apxfig:cmp-single-split-ffn}
\end{figure}

\section{Distance and Relative Positional Encoding (RPE)}
\label{apx:rpe}
In Section~\ref{sec:single-channel}, we find that the single-channel update (Equation~\ref{eq:channel-c}) in probabilistic transformers is almost identical to scaled dot-product attention in transformers. This observation is based on the hypothesis that probabilistic transformers and transformers are sharing the same positional encoding method. But this is not the case.

In section~\ref{sec:distance}, we mention that to capture the word order information, we use a clip function to select the ternary potential function based on the distance of two words (Equation~\ref{eq:distance}). This is similar to the relative positional encoding (RPE) in transformers. \citet{shaw-etal-2018-self} proposes a method to add an additional component to key and value, based on the clipped distance. Specifically, the scaled dot-product attention with RPE could be rewritten as
\begin{equation*}
     e_{ij} = \frac{x_{i}W^{Q}\left(x_{j}W^{K} + a^{K}_{ij}\right)^{T}}{\sqrt{d_{k}}} 
\end{equation*}
\begin{equation*}
     z_{i} = \sum^{n}_{j=1}\alpha_{ij}\left(x_{j}W^{V} + a_{ij}^{V}\right)
\end{equation*}
where $x_i$ is the input representation of the $i$-th word, $z_i$ is the output representation, $\alpha_{ij} = \frac{\exp e_{ij}}{\sum_k \exp e_{ik}}$. The additional component is a learnable parameter that based on the clipped distance
\begin{align*}
a^K_{ij} &= w^K_{\mathrm{clip}(j - i, k)} \\
a^V_{ij} &= w^V_{\mathrm{clip}(j - i, k)} \\
\mathrm{clip}(x, k) &= \max(-k, \min(k, x))
\end{align*}

For probabilistic transformers, we directly add the distance information to the ternary potential function. Combining Equation~\ref{eq:distance} and \ref{eq:channel-c}, we could rewrite the single-channel update as
\begin{equation*}
     e_{ij} = \frac{x_{i}\mathbf{U}_{ij}\left(x_{j}\mathbf{V}_{ij}\right)^{T}}{\lambda_H} 
\end{equation*}
\begin{equation*}
     z_{i} = \sum^{n}_{j=1}\alpha_{ij}\left(x_{j}\mathbf{V}_{ij}\right)
\end{equation*}
where $\alpha_{ij} = \frac{\exp e_{ij}}{\sum_k \exp e_{ik}}$. The weights are based on the clip function $f$ in Equation~\ref{eq:clip}
\begin{align*}
\mathbf{U}_{ij} &= \mathbf{U}[f(i - j)] \\
\mathbf{V}_{ij} &= \mathbf{V}[f(i - j)]
\end{align*}

Notice that this way of positional encoding is quite parameter inefficient. It also makes our training process much slower than that of transformers.

\section{Details for Tasks and Datasets}
\label{apx:task-dataset}

In this section, we will introduce our tasks and datasets in detail. A brief introduction is shown in Section~\ref{sec:task-dataset}.

\subsection{Masked Language Modeling}
\label{apx:exp-mlm}

Masked Language Modeling (MLM) tasks generally evaluate the expressiveness of contextural word representations. We perform MLM tasks on two corpora: the Penn TreeBank (PTB) and Brown Laboratory for Linguistic Information Processing (BLLIP). We randomly replace words with a mask token \texttt{<mask>} at a rate of 30\% and the model is required to predict the original word. Following \citet{shen-etal-2022-unsupervised}, we never mask \texttt{<unk>} tokens. The performance of MLM is evaluated by measuring perplexity (lower is better) on masked words.

\textbf{PTB}. The Penn Treebank \cite{marcinkiewicz1994building}, in particular the sections of the corpus corresponding to the articles of Wall Street Journal (WSJ), is a standard dataset for language modeling \cite{mikolov2012statistical} and sequence labeling \cite{dinarelli2019seq2biseq}. Following the setting in \citet{shen-etal-2021-structformer}, we use the preprocessing method proposed in \citet{mikolov2012statistical}. It removes all punctuation and replaces low-frequency words with \texttt{<unk>}. The processed dataset has a vocabulary size of 10000, including \texttt{<unk>} and \texttt{<mask>}.

\textbf{BLLIP}. The Brown Laboratory for Linguistic Information Processing dataset \cite{charniak2000bllip} is a large corpus similar to the PTB dataset in style. The entire dataset contains 24 million sentences from Wall Street Journal. In our experiments, we only use a small subset of this corpus. Following the same setting as \citet{shen-etal-2022-unsupervised}, we use the BLLIP-XS split proposed in \citet{hu-etal-2020-systematic} with around 40k sentences and 1M tokens as the train set. The validation set consists of the first section each year and the test set consists of the second section each year. We remove all punctuation, replace numbers with a single character \texttt{N} and use lower-case letters. The vocabulary contains words
that appear more than 27 times in the entire BLLIP dataset, with size 30231 including \texttt{<unk>} and \texttt{<mask>}.

\subsection{Sequence Labeling}

Sequence labeling tasks require models to predict the tag for each word in the sequence. For sequence labeling tasks, we perform part-of-speech (POS) tagging on two datasets: the Penn TreeBank (PTB) and the Universal Dependencies (UD). We also perform named entity recognition (NER) on CoNLL-2003.

\textbf{PTB}. As introduced in Appendix~\ref{apx:exp-mlm}, we also use the PTB dataset for POS tagging but with a different setting. We use the most commons split of this corpus for POS tagging, where sections from 0 to 18 are used as the train set, sections from 19 to 21 are used as the validation set, and sections from 22 to 24 are used as the test set. All words in the train set compose the vocabulary.

\textbf{UD}. UD is a project that develops cross-linguistically consistent treebank annotation for many languages \cite{de2021universal}. We test our model on the language-specific part-of-speech (XPOS) tags of the English EWT dataset with the standard splits. All words in the train set compose the vocabulary.

\textbf{CoNLL-2003}. It is a named entity recognition dataset which is released as part of CoNLL-2003 shared task \cite{tjong-kim-sang-de-meulder-2003-introduction}. We test our model on the English dataset. All words in the train set compose the vocabulary. We only project the final word representation of each word to the tag set with the BIOES scheme without using a CRF decoder.

\subsection{Text Classification}

Text Classification tasks need to classify sentences into different classes. We use the Stanford Sentiment Treebank (SST) \cite{socher-etal-2013-recursive} as the dataset. It has two variants: binary classification (SST-2) and fine-grained classification (SST-5). The dataset comes from SentEval \cite{conneau2018senteval}.

\textbf{SST-2}. SST-2 classifies each movie review into positive or negative classes. It contains 67k sentences in the train set.

\textbf{SST-5}. SST-5 classifies sentences into 5 classes: negative, somewhat negative, neutral, somewhat positive and positive. It contains 8.5k sentences in the train set.

In text classification, all words in the train set compose the vocabulary.

\subsection{Syntactic Test}

To evaluate the compositional generalization abilities of our model, we perform a syntactic test on the COGS \cite{kim-linzen-2020-cogs} dataset. COGS is a semantic parsing dataset that measures the compositional generalization abilities of models. We follow the settings in \citet{ontanon2021making}, which turns the task from seq2seq into a sequence tagging task. The model needs to predict 5 tags for each input word: a \textit{parent} word, the \textit{role} of the relation between the word and its parent (if applicable), the \textit{category}, the \textit{noun determiner} (for nouns) and the \textit{verb name} (for verbs). With these tags, one can reconstruct the original output deterministically.

For \textit{role}, \textit{category}, \textit{noun determiner} and \textit{verb name}, we directly project word representations to each tag set. For the \textit{parent} tag, \cite{ontanon2021making} propose 3 types of prediction heads:

\begin{itemize}
  \item \textit{Absolute} uses a direct projection to predict the absolute index of the parent in the input sequence (-1 for no parent).
  \item \textit{Relative} uses a direct projection to predict the relative offset of the parent token with respect to the current token, or self for no parent.
  \item \textit{Attention} uses the attention weights from a new attention layer with a single head to predict the parent.
\end{itemize}

We empirically find that \textit{relative} performs the best in most settings for both transformers and probabilistic transformers. This is not consistent with the observations in \citet{ontanon2021making} who finds that \textit{attention} outperforms other settings. We still apply the \textit{relative} setting in our experiments.

\section{Hyperparameters and Implementation}

We report our hyperparameters in Table~\ref{tab:hyper-pt} for probabilistic transformers and Table~\ref{tab:hyper-tsfm} for transformers. We tune the models for each task except the syntactic test through random search. We run experiments on one NVIDIA GeForce RTX 2080 Ti and all the experiments could finish in one day. Our implementation is based on the flair framework \cite{akbik-etal-2019-flair}.

\begin{table*}[]
  \centering
  \begin{tabular}{@{}lccccccc@{}}
  \toprule
  \multicolumn{1}{c}{\multirow{2}{*}{\textbf{Probabilistic Transformer}}} & \multicolumn{2}{c}{MLM} & \multicolumn{2}{c}{POS} & \multicolumn{2}{c}{CLS} & SYN    \\
  \multicolumn{1}{c}{}                                                    & PTB        & BLLIP      & PTB        & UD         & SST-2       & SST-5     & COGS   \\ \midrule
  Label set size $d$                                                      & 384        & 384        & 128        & 128        & 512         & 256       & 64     \\
  Root label set size $d_{root}$                                          & --         & --         & --         & --         & 1024        & 512       & --     \\
  \# of channels $h$                                                      & 16         & 16         & 12         & 18         & 10          & 18        & 4      \\
  \# of iterations $T$                                                    & 5          & 5          & 3          & 2          & 1           & 4         & 2      \\
  Distance threshold $\gamma$                                             & 3          & 3          & 3          & 3          & 3           & 3         & 8      \\
  Decomposition                                                           & UV         & UV         & UV         & --         & UV          & UVW       & UV     \\
  Decomposition rank $r$                                                  & 64         & 64         & 128        & --         & 64          & 64        & 16     \\
  Dropout                                                                 & 0.15       & 0.15       & 0.05       & 0.1        & 0.1         & 0.05      & 0.1    \\
  Asynchronous update                                                     & \multicolumn{7}{c}{Yes}                                                              \\
  Learning rate                                                           & 0.001      & 0.001      & 0.0024     & 0.0062     & 0.0001      & 0.0002    & 0.0025 \\
  Weight decay                                                     & 1.4e-6     & 1.4e-6     & 8e-6       & 2.2e-6     & 3e-7        & 3e-7      & 1e-9   \\
  L2 reg for $\mathbf{T}$                                                 & 5e-4       & 5e-4       & 0          & 4e-4       & 0           & 0         & 0      \\ \bottomrule
  \end{tabular}
  \caption{Hyperparameters for probabilistic transformers in our experiments.}
  \label{tab:hyper-pt}
\end{table*}

\begin{table*}[]
  \centering
  \begin{tabular}{@{}lccccccc@{}}
  \toprule
  \multicolumn{1}{c}{\multirow{2}{*}{\textbf{Transformer}}} & \multicolumn{2}{c}{MLM} & \multicolumn{2}{c}{POS} & \multicolumn{2}{c}{CLS} & SYN    \\
  \multicolumn{1}{c}{}                                      & PTB        & BLLIP      & PTB         & UD        & SST-2       & SST-5     & COGS   \\ \midrule
  Embedding size $d_{model}$                                & 384        & 256        & 512         & 384       & 256         & 128       & 64     \\
  FFN inner layer size $d_{ff}$                             & 2048       & 2048       & 2048        & 512       & 512         & 1024      & 256    \\
  \# of heads $h$                                           & 8          & 14         & 14          & 14        & 10          & 14        & 4      \\
  \# of layers $N$                                          & 5          & 4          & 5           & 4         & 8           & 4         & 2      \\
  Positional Encoding                                       & abs        & abs        & abs         & abs       & abs         & abs       & rel-8  \\
  Head dimension $d_{qkv}$                                  & 256        & 128        & 32          & 16        & 256         & 256       & 16     \\
  Dropout                                                   & 0.15       & 0.15       & 0.15        & 0         & 0.05        & 0         & 0.1    \\
  Learning rate                                             & 0.0001     & 0.0002     & 0.0004      & 0.0004    & 0.0001      & 0.0002    & 0.0005 \\
  Weight decay                                       & 1.2e-6     & 3.5e-6     & 3.2e-6      & 1.4e-6    & 1.9e-6      & 2.7e-6    & 1e-9   \\ \bottomrule
  \end{tabular}
  \caption{Hyperparameters for transformers in our experiments.}
  \label{tab:hyper-tsfm}
\end{table*}

\section{Case Studies of Learned Dependency Structures}
\label{apx:case-study}

A probabilistic transformer infers marginal distributions over both $Z$ and $H$ variables, the latter of which can be used to extract a dependency structure. Since our model is trained on downstream tasks such as MLM without access to gold parse trees, it can be seen as performing unsupervised dependency parsing.
We visualize the dependency structures learned by a probabilistic transformer by looking at the most probable head of each word in the sentence.

Figure~\ref{apxfig:example-dep} illustrates the dependency structures extracted from a probabilistic transformer trained on the PTB dataset under the MLM task. The sentence comes from the test set of the PTB dataset. We show the head of each word in all the channels. The numbers on the dependency arcs represent probabilities estimated by the model. The model does not contain a root node, so there is at least one circle in the dependency graph.

From the figure, we can see that our model is very confident in its choices of dependency arcs, with all the probabilities close to 1, which indicates strong compatibilities between the latent representations of connected word pairs. The predicted structure somewhat makes sense. For example, it puts `she said' together. But generally, most of the dependency arcs are not consistent with human-designed dependency relations.




\begin{figure*}
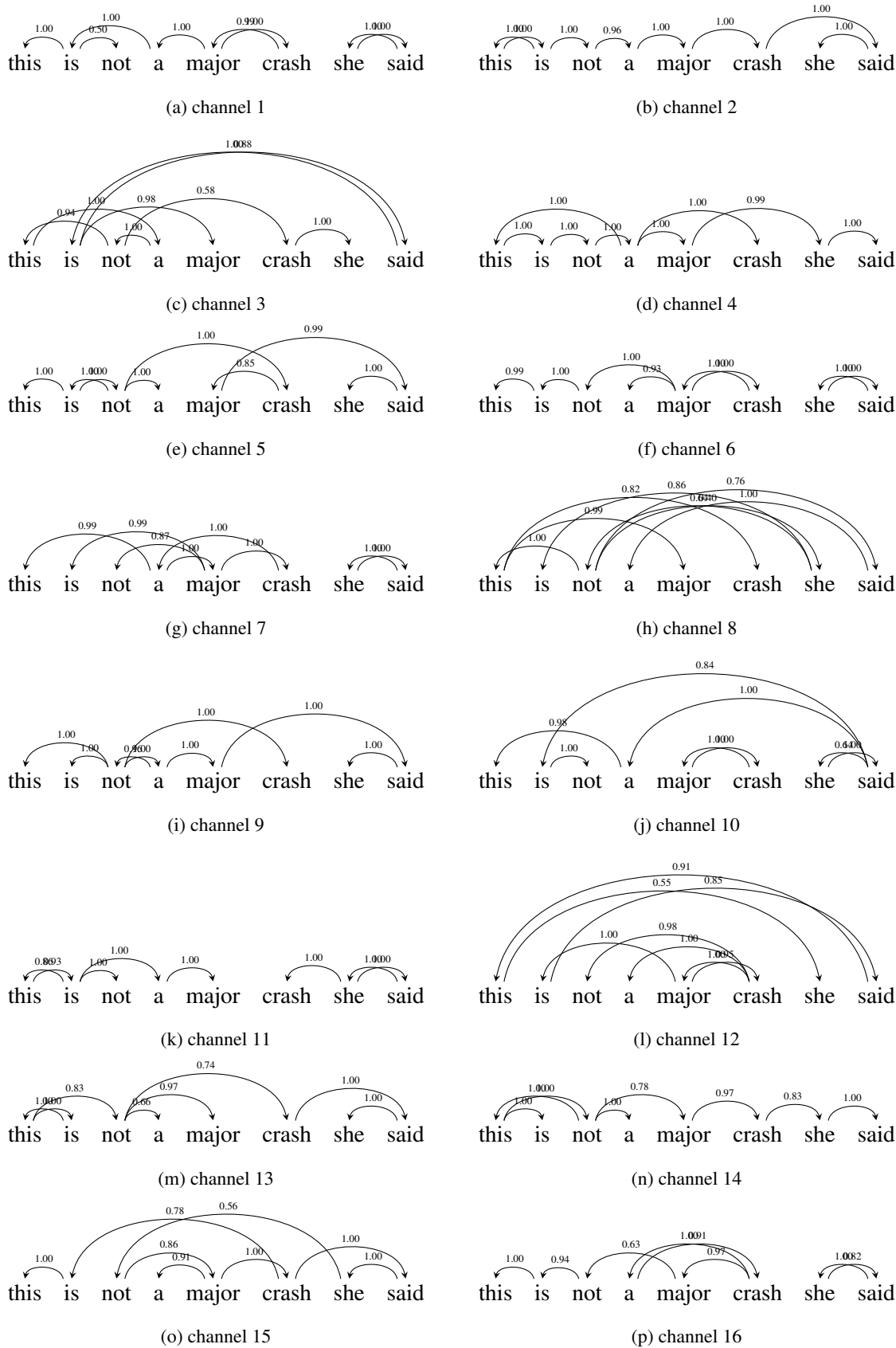

  \centering
  \begin{subfigure}{.48\textwidth}
    \centering
    \begin{dependency}[theme = simple]
      \begin{deptext}[column sep=.5em]
          this \& is \& not \& a \& major \& crash \& she \& said \\
      \end{deptext}
      \depedge{2}{1}{1.00}
      \depedge{4}{2}{1.00}
      \depedge{2}{3}{0.50}
      \depedge{5}{4}{1.00}
      \depedge{6}{5}{0.99}
      \depedge{5}{6}{1.00}
      \depedge{8}{7}{1.00}
      \depedge{7}{8}{1.00}
    \end{dependency}
    \caption{channel 1}
  \end{subfigure}
  \begin{subfigure}{.48\textwidth}
    \centering
    \begin{dependency}[theme = simple]
      \begin{deptext}[column sep=.5em]
          this \& is \& not \& a \& major \& crash \& she \& said \\
      \end{deptext}
      \depedge{2}{1}{1.00}
      \depedge{1}{2}{1.00}
      \depedge{2}{3}{1.00}
      \depedge{3}{4}{0.96}
      \depedge{4}{5}{1.00}
      \depedge{5}{6}{1.00}
      \depedge{8}{7}{1.00}
      \depedge{6}{8}{1.00}
    \end{dependency}
    \caption{channel 2}
  \end{subfigure}
  \begin{subfigure}{.48\textwidth}
    \centering
    \begin{dependency}[theme = simple]
      \begin{deptext}[column sep=.5em]
          this \& is \& not \& a \& major \& crash \& she \& said \\
      \end{deptext}
      \depedge{3}{1}{0.94}
      \depedge{8}{2}{1.00}
      \depedge{4}{3}{1.00}
      \depedge{1}{4}{1.00}
      \depedge{2}{5}{0.98}
      \depedge{3}{6}{0.58}
      \depedge{6}{7}{1.00}
      \depedge{2}{8}{0.88}
    \end{dependency}
    \caption{channel 3}
  \end{subfigure}
  \begin{subfigure}{.48\textwidth}
    \centering
    \begin{dependency}[theme = simple]
      \begin{deptext}[column sep=.5em]
          this \& is \& not \& a \& major \& crash \& she \& said \\
      \end{deptext}
      \depedge{4}{1}{1.00}
      \depedge{1}{2}{1.00}
      \depedge{2}{3}{1.00}
      \depedge{3}{4}{1.00}
      \depedge{4}{5}{1.00}
      \depedge{4}{6}{1.00}
      \depedge{5}{7}{0.99}
      \depedge{7}{8}{1.00}
    \end{dependency}
    \caption{channel 4}
  \end{subfigure}
  \begin{subfigure}{.48\textwidth}
    \centering
    \begin{dependency}[theme = simple]
      \begin{deptext}[column sep=.5em]
          this \& is \& not \& a \& major \& crash \& she \& said \\
      \end{deptext}
      \depedge{2}{1}{1.00}
      \depedge{3}{2}{1.00}
      \depedge{2}{3}{1.00}
      \depedge{3}{4}{1.00}
      \depedge{6}{5}{0.85}
      \depedge{3}{6}{1.00}
      \depedge{8}{7}{1.00}
      \depedge{5}{8}{0.99}
    \end{dependency}
    \caption{channel 5}
  \end{subfigure}
  \begin{subfigure}{.48\textwidth}
    \centering
    \begin{dependency}[theme = simple]
      \begin{deptext}[column sep=.5em]
          this \& is \& not \& a \& major \& crash \& she \& said \\
      \end{deptext}
      \depedge{2}{1}{0.99}
      \depedge{3}{2}{1.00}
      \depedge{5}{3}{1.00}
      \depedge{5}{4}{0.93}
      \depedge{6}{5}{1.00}
      \depedge{5}{6}{1.00}
      \depedge{8}{7}{1.00}
      \depedge{7}{8}{1.00}
    \end{dependency}
    \caption{channel 6}
  \end{subfigure}
  \begin{subfigure}{.48\textwidth}
    \centering
    \begin{dependency}[theme = simple]
      \begin{deptext}[column sep=.5em]
          this \& is \& not \& a \& major \& crash \& she \& said \\
      \end{deptext}
      \depedge{4}{1}{0.99}
      \depedge{5}{2}{0.99}
      \depedge{5}{3}{0.87}
      \depedge{6}{4}{1.00}
      \depedge{4}{5}{1.00}
      \depedge{5}{6}{1.00}
      \depedge{8}{7}{1.00}
      \depedge{7}{8}{1.00}
    \end{dependency}
    \caption{channel 7}
  \end{subfigure}
  \begin{subfigure}{.48\textwidth}
    \centering
    \begin{dependency}[theme = simple]
      \begin{deptext}[column sep=.5em]
          this \& is \& not \& a \& major \& crash \& she \& said \\
      \end{deptext}
      \depedge{3}{1}{1.00}
      \depedge{7}{2}{0.86}
      \depedge{7}{3}{0.64}
      \depedge{8}{4}{1.00}
      \depedge{1}{5}{0.99}
      \depedge{1}{6}{0.82}
      \depedge{3}{7}{0.40}
      \depedge{3}{8}{0.76}
    \end{dependency}
    \caption{channel 8}
  \end{subfigure}
  \begin{subfigure}{.48\textwidth}
    \centering
    \begin{dependency}[theme = simple]
      \begin{deptext}[column sep=.5em]
          this \& is \& not \& a \& major \& crash \& she \& said \\
      \end{deptext}
      \depedge{3}{1}{1.00}
      \depedge{3}{2}{1.00}
      \depedge{4}{3}{0.96}
      \depedge{3}{4}{1.00}
      \depedge{4}{5}{1.00}
      \depedge{3}{6}{1.00}
      \depedge{8}{7}{1.00}
      \depedge{5}{8}{1.00}
    \end{dependency}
    \caption{channel 9}
  \end{subfigure}
  \begin{subfigure}{.48\textwidth}
    \centering
    \begin{dependency}[theme = simple]
      \begin{deptext}[column sep=.5em]
          this \& is \& not \& a \& major \& crash \& she \& said \\
      \end{deptext}
      \depedge{4}{1}{0.98}
      \depedge{8}{2}{0.84}
      \depedge{2}{3}{1.00}
      \depedge{8}{4}{1.00}
      \depedge{6}{5}{1.00}
      \depedge{5}{6}{1.00}
      \depedge{8}{7}{0.64}
      \depedge{7}{8}{1.00}
    \end{dependency}
    \caption{channel 10}
  \end{subfigure}
  \begin{subfigure}{.48\textwidth}
    \centering
    \begin{dependency}[theme = simple]
      \begin{deptext}[column sep=.5em]
          this \& is \& not \& a \& major \& crash \& she \& said \\
      \end{deptext}
      \depedge{2}{1}{0.86}
      \depedge{1}{2}{0.93}
      \depedge{2}{3}{1.00}
      \depedge{2}{4}{1.00}
      \depedge{4}{5}{1.00}
      \depedge{7}{6}{1.00}
      \depedge{8}{7}{1.00}
      \depedge{7}{8}{1.00}
    \end{dependency}
    \caption{channel 11}
  \end{subfigure}
  \begin{subfigure}{.48\textwidth}
    \centering
    \begin{dependency}[theme = simple]
      \begin{deptext}[column sep=.5em]
          this \& is \& not \& a \& major \& crash \& she \& said \\
      \end{deptext}
      \depedge{8}{1}{0.91}
      \depedge{5}{2}{1.00}
      \depedge{6}{3}{0.98}
      \depedge{6}{4}{1.00}
      \depedge{6}{5}{1.00}
      \depedge{5}{6}{0.95}
      \depedge{1}{7}{0.55}
      \depedge{2}{8}{0.85}
    \end{dependency}
    \caption{channel 12}
  \end{subfigure}
  \begin{subfigure}{.48\textwidth}
    \centering
    \begin{dependency}[theme = simple]
      \begin{deptext}[column sep=.5em]
          this \& is \& not \& a \& major \& crash \& she \& said \\
      \end{deptext}
      \depedge{2}{1}{1.00}
      \depedge{1}{2}{1.00}
      \depedge{1}{3}{0.83}
      \depedge{3}{4}{0.66}
      \depedge{3}{5}{0.97}
      \depedge{3}{6}{0.74}
      \depedge{8}{7}{1.00}
      \depedge{6}{8}{1.00}
    \end{dependency}
    \caption{channel 13}
  \end{subfigure}
  \begin{subfigure}{.48\textwidth}
    \centering
    \begin{dependency}[theme = simple]
      \begin{deptext}[column sep=.5em]
          this \& is \& not \& a \& major \& crash \& she \& said \\
      \end{deptext}
      \depedge{3}{1}{1.00}
      \depedge{1}{2}{1.00}
      \depedge{1}{3}{1.00}
      \depedge{3}{4}{1.00}
      \depedge{3}{5}{0.78}
      \depedge{5}{6}{0.97}
      \depedge{6}{7}{0.83}
      \depedge{7}{8}{1.00}
    \end{dependency}
    \caption{channel 14}
  \end{subfigure}
  \begin{subfigure}{.48\textwidth}
    \centering
    \begin{dependency}[theme = simple]
      \begin{deptext}[column sep=.5em]
          this \& is \& not \& a \& major \& crash \& she \& said \\
      \end{deptext}
      \depedge{2}{1}{1.00}
      \depedge{6}{2}{0.78}
      \depedge{7}{3}{0.56}
      \depedge{5}{4}{0.91}
      \depedge{3}{5}{0.86}
      \depedge{5}{6}{1.00}
      \depedge{8}{7}{1.00}
      \depedge{6}{8}{1.00}
    \end{dependency}
    \caption{channel 15}
  \end{subfigure}
  \begin{subfigure}{.48\textwidth}
    \centering
    \begin{dependency}[theme = simple]
      \begin{deptext}[column sep=.5em]
          this \& is \& not \& a \& major \& crash \& she \& said \\
      \end{deptext}
      \depedge{2}{1}{1.00}
      \depedge{3}{2}{0.94}
      \depedge{5}{3}{0.63}
      \depedge{6}{4}{1.00}
      \depedge{6}{5}{0.97}
      \depedge{4}{6}{0.91}
      \depedge{8}{7}{1.00}
      \depedge{7}{8}{0.82}
    \end{dependency}
    \caption{channel 16}
  \end{subfigure}
  \caption{Dependency structures learned by a probabilistic transformer under the MLM task. The numbers on the dependency arcs represent the confidence of the head word.}
  \label{apxfig:example-dep}
\end{figure*}

\end{document}